\documentclass{article}
\usepackage{iclr2027_conference,times}
\usepackage{graphicx}
\usepackage{booktabs}
\usepackage{amsmath}
\usepackage{amssymb}
\usepackage{mathtools}
\usepackage{xcolor}
\usepackage{colortbl}
\usepackage{microtype}
\usepackage{hyperref}
\usepackage{url}
\usepackage{fontawesome5}
\usepackage{array}
\usepackage{tabularx}
\usepackage{adjustbox}
\usepackage{amsthm}
\newtheorem{proposition}{Proposition}
\newtheorem{corollary}{Corollary}
\usepackage{multirow}
\usepackage{subcaption}
\usepackage{placeins}
\usepackage{wrapfig}
\usepackage{needspace}

\usepackage[most]{tcolorbox}
\hypersetup{
  hidelinks,
  pdfauthor={Junyu Guo, Yuchen Fang, Shangding Gu, Costas Spanos, James Demmel, Javad Lavaei},
  pdftitle={When Context Changes: Understanding Update Failures in LLMs}
}
\usepackage{enumitem}

\iclrfinalcopy

\definecolor{staleink}{HTML}{7B4FA0}
\definecolor{stalefill}{HTML}{E4D6F2}
\definecolor{currentink}{HTML}{2E7D5B}
\definecolor{currentfill}{HTML}{CFE6DA}
\definecolor{crossink}{HTML}{595959}
\definecolor{crossfill}{HTML}{EDEDED}
\definecolor{neutralink}{HTML}{46515D}
\definecolor{neutralfill}{HTML}{F1F2F3}
\definecolor{systemturnfill}{HTML}{ECEFF1}
\definecolor{systemturnink}{HTML}{59636D}
\definecolor{userturnfill}{HTML}{E8F2F8}
\definecolor{userturnink}{HTML}{326F8A}
\definecolor{assistantturnfill}{HTML}{F5F0E6}
\definecolor{assistantturnink}{HTML}{8A7A55}

\newtcolorbox{contextbox}[1]{enhanced,breakable,colback=neutralfill,colframe=neutralink,
  boxrule=0.7pt,arc=1mm,left=1.5mm,right=1.5mm,top=1mm,bottom=1mm,
  title={#1},fonttitle=\bfseries\small}
\newtcolorbox{stalebox}[1]{enhanced,breakable,colback=stalefill,colframe=staleink,
  boxrule=0.8pt,arc=1mm,left=1.5mm,right=1.5mm,top=1mm,bottom=1mm,
  title={#1},fonttitle=\bfseries\small}
\newtcolorbox{currentbox}[1]{enhanced,breakable,colback=currentfill,colframe=currentink,
  boxrule=0.8pt,arc=1mm,left=1.5mm,right=1.5mm,top=1mm,bottom=1mm,
  title={#1},fonttitle=\bfseries\small}
\newtcolorbox{crossbox}[1]{enhanced,breakable,colback=crossfill,colframe=crossink,
  boxrule=0.8pt,arc=1mm,left=1.5mm,right=1.5mm,top=1mm,bottom=1mm,
  title={#1},fonttitle=\bfseries\small}
\newtcolorbox{findingbox}[1]{enhanced,colback=white,colframe=neutralink,
  boxrule=1pt,arc=1mm,left=2mm,right=2mm,top=1.5mm,bottom=1.5mm,
  title={#1},fonttitle=\bfseries}
\newtcolorbox{systemturn}{enhanced,colback=systemturnfill,colframe=systemturnink,
  boxrule=0.5pt,arc=0.8mm,left=1.2mm,right=1.2mm,top=0.7mm,bottom=0.7mm,
  before skip=0.6mm,after skip=0.6mm,fontupper=\footnotesize}
\newtcolorbox{userturn}{enhanced,colback=userturnfill,colframe=userturnink,
  boxrule=0.5pt,arc=0.8mm,left=1.2mm,right=1.2mm,top=0.7mm,bottom=0.7mm,
  before skip=0.6mm,after skip=0.6mm,fontupper=\footnotesize}
\newtcolorbox{assistantturn}{enhanced,colback=assistantturnfill,colframe=assistantturnink,
  boxrule=0.5pt,arc=0.8mm,left=1.2mm,right=1.2mm,top=0.7mm,bottom=0.7mm,
  before skip=0.6mm,after skip=0.6mm,fontupper=\footnotesize}
\newtcolorbox{omittedturn}{enhanced,colback=white,colframe=neutralfill,
  boxrule=0.4pt,arc=0.8mm,left=1.2mm,right=1.2mm,top=0.5mm,bottom=0.5mm,
  borderline={0.4pt}{0pt}{neutralink,dashed},before skip=0.6mm,after skip=0.6mm,
  fontupper=\footnotesize\itshape,colupper=neutralink}

\newcolumntype{Y}{>{\raggedright\arraybackslash}X}

\title{When Context Changes: Understanding Update Failures in LLMs}

\author{%
\normalfont
\setlength{\tabcolsep}{2.5pt}
\begin{tabular}{@{}ccc@{}}
\textbf{Junyu Guo} & \textbf{Yuchen Fang} & \textbf{Shangding Gu} \\
{\small\fontencoding{T1}\selectfont\texttt{junyuguo24@berkeley.edu}} &
{\small\fontencoding{T1}\selectfont\texttt{yc\_fang@berkeley.edu}} &
{\small\fontencoding{T1}\selectfont\texttt{shangding.gu@berkeley.edu}} \\[1ex]
\textbf{Costas Spanos} & \textbf{James Demmel} & \textbf{Javad Lavaei} \\
{\small\fontencoding{T1}\selectfont\texttt{spanos@eecs.berkeley.edu}} &
{\small\fontencoding{T1}\selectfont\texttt{demmel@berkeley.edu}} &
{\small\fontencoding{T1}\selectfont\texttt{lavaei@berkeley.edu}} \\[1ex]
\multicolumn{3}{c}{University of California, Berkeley} \\[0.8ex]
\multicolumn{3}{c}{%
  \small
  \href{https://github.com/JamesJunyuGuo/CICM}{%
    \faGithub\enspace
    \texttt{github.com/JamesJunyuGuo/CICM}%
  }%
}
\end{tabular}%
}

\begin{document}
\raggedbottom
\maketitle
\fancyhead{}
\renewcommand{\headrulewidth}{0pt}

\begin{abstract}
As preferences, goals, and facts change, LLM agents must use the current state
while earlier versions remain in context. Yet they can answer with an old
value of the same variable, a failure that we call \textbf{stale binding}. To study when models use outdated information and why, we introduce
Controlled In-Context Memory (CICM), a benchmark for tracking and using
updated information in conversations and agent logs. We observe that even frontier
reasoning models can fail to recover the current state. We find that in open-source models probes can still recover the updated value when the model answers
with an old one, pointing to a failure to select information that remains
available. Component tests in Qwen and Pythia identify a mechanism for this
selection failure: attention drift, where attention favors old values
over the current one when producing an answer. We study a one-layer transformer to mathematically understand how this phenomenon happens: when attention scores are similar, several old values can together
receive more attention than the current value. Guided by this explanation,
we redirect attention toward the current value without further training.
When the current value is requested directly, adjusting this intervention for each input corrects
most old-value errors across various model families while preserving nearly all
initially correct answers. Reliable context management therefore requires more than remembering
updated information: models must use it to guide their answers.
\end{abstract}

\let\cicmSavedTopFraction\topfraction
\let\cicmSavedBottomFraction\bottomfraction
\let\cicmSavedTextFraction\textfraction
\newlength{\cicmSavedTextFloatSep}
\setlength{\cicmSavedTextFloatSep}{\textfloatsep}
\setlength{\textfloatsep}{12pt plus 2pt minus 2pt}
\setcounter{topnumber}{1}
\setcounter{bottomnumber}{1}
\setcounter{totalnumber}{2}
\renewcommand{\topfraction}{0.8}
\renewcommand{\bottomfraction}{0.55}
\renewcommand{\textfraction}{0.2}

\section{Introduction}

 LLM agents operate in dynamic environments where context continuously evolves. During an interaction, users revise preferences, switch objectives, and update facts that guide future actions,
such as an address or a project deadline. A planning agent must follow the
revised goal, and  a coding agent must use the current requirements. However, earlier
versions remain in the conversation or execution log alongside their updates.
Reliable context management therefore requires identifying which state is
currently valid and using it to guide the next response.

Prior work examines context management through retrieval, memory, and forgetting.
Long-context studies test  information access at different positions
\citep{liu2024lost}; conversational benchmarks evaluate remembering facts,
tracking updates, and following preferences
\citep{wu2025longmemeval,zhao2025prefeval}. ICF-Bench tests selective forgetting
of information that should no longer guide a response \citep{icfbench2026},
while interference studies show that repeated updates to the same key impair
recall \citep{wang2026unable,chattaraj2026dual}. These findings motivate a
closer diagnosis of update errors: what information does a wrong answer use,
and why does it win over the current state? We investigate this question by giving models successive updates to the same
variable and asking for its current value, with the full history available.
For example, if a meal preference changes from gluten-free to Mediterranean, answering
\emph{gluten-free} returns an old value of the right variable; answering a
music preference confuses variables. We call the former error \emph{stale
binding}. In controlled experiments, repeated updates predominantly produce
old-value errors even when equally long contexts without updates remain easy.
The problem also reaches complex operational histories: GPT-5.6 Sol and Claude
Opus 4.8 produce incorrect current-state answers after completing their
reasoning, while answering every paired current-state snapshot correctly.

We ask how these errors arise, whether the current information remains
available inside the model, and whether changing how it is selected can repair
the answer. To connect these questions, we introduce \emph{Controlled
In-Context Memory (CICM)}, a benchmark with programmatic state labels and error
categories that support systematic diagnosis beyond accuracy. Its preference
core builds on PrefEval content and supports probes and interventions. It also includes
new dialogues, decisions under changing constraints, and operational logs
extend the same update problem to richer tasks. This design allows us to connect
behavioral errors to internal measurements, a simple mathematical explanation,
and targeted interventions (Figure~\ref{fig:overview}).

\begin{figure}[tb]
  \centering
  \includegraphics[width=\textwidth]{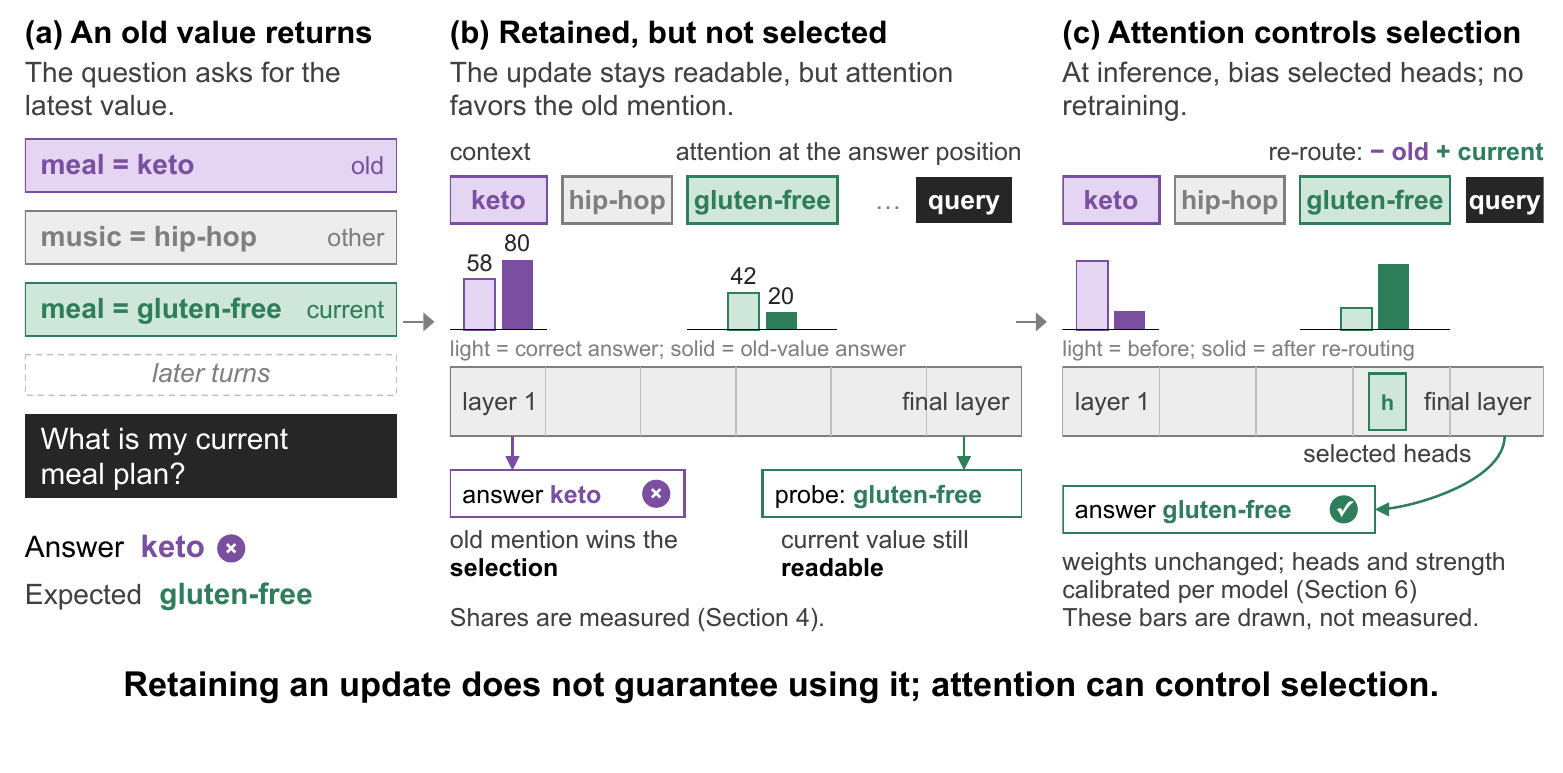}
\caption{\textbf{a--b}, An old value can win while the current value remains
readable. The bars in \textbf{b} are measured attention shares at the answer
position (Section~\ref{sec:mechanism}). \textbf{c}, Redirecting attention can
correct selection without retraining.}
\label{fig:overview}
\end{figure}

Overall, our contributions can be summarized as follows:
\begin{enumerate}[leftmargin=*,itemsep=1pt]
\item \textbf{A diagnostic benchmark for in-context updates.} CICM distinguishes
old-value reuse from other-variable confusion across dialogue and agent-log
tasks, connecting controlled mechanism tests with complex histories that
challenge frontier models.
\item \textbf{Selection failures and attention drift.} Probes recover current
values even when  old ones are chosen. Measurements and component
interventions in various model families identify \emph{attention drift}: attention
favors old values over the current one, affecting which value reaches the answer.
\item \textbf{A one-layer transformer explanation.} We show how position shapes
attention scores and how several old values can together receive more attention
than the current value, making the competition behind selection explicit.
\item \textbf{Training-free repair at test time.} Redirecting attention improves
direct current-value answers across Qwen, Llama, Mistral, and Gemma without
changing model weights. Adjusting the intervention for each input corrects
most old-value errors while preserving nearly all initially correct answers.
\end{enumerate}

\section{The Stale-Binding Phenomenon}
\label{sec:phenomenon}
\label{sec:behavioral-evidence}

A binding pairs a variable with a value. When context assigns the same
variable more than once, the latest explicit assignment gives its current
value while earlier ones remain as old values, still visible to the model.
An \emph{old-value error}, or \emph{stale binding}, returns an old value of the
queried variable instead of its current one. We first isolate this error using synthetic overwrite tasks, and then examine
richer dialogue and operational updates. Answers are classified as current,
old, another variable's value, or others. Section~\ref{sec:cicm} introduces
CICM's construction. We begin with a simple overwrite task: a target variable is repeatedly
reassigned while unrelated lines control context length. Qwen2.5-7B remains
nearly perfect without overwrites, but accuracy declines as more old values
compete with the current one at matched lengths (Figure~\ref{fig:phenomenon}a).
Among 333 failures, 331 return an old value of the queried variable; two copy
another in-context value, and none introduces an absent value. The most recent old value is overrepresented relative to a uniform
choice among old values at $k=2,4,8$ (Appendix~\ref{app:which-old}), here $k$ refers to the number of the old values. Also, a parallel-variable control shows negligible cross-variable leakage.
(Appendix~\ref{app:synthetic-controls}).

Scale postpones the failure without changing its nature. 70B-class models and
GPT-4o tolerate more revisions but still fail at $k=64$, while every matched
320-line prompt without overwrites is answered correctly
(Figure~\ref{fig:phenomenon}b): the difficulty comes from the updates, not from
the context length. The error also keeps its signature. Qwen2.5-72B and GPT-4o
return old values on every error, Llama-3.1-70B shows a more varied mixture, and
recent old values are again favored, as all seven GPT-4o errors select one of the
two most recent (Appendix~\ref{app:which-old}).
\begin{figure}[tb]
  \centering
  \includegraphics[width=\textwidth]{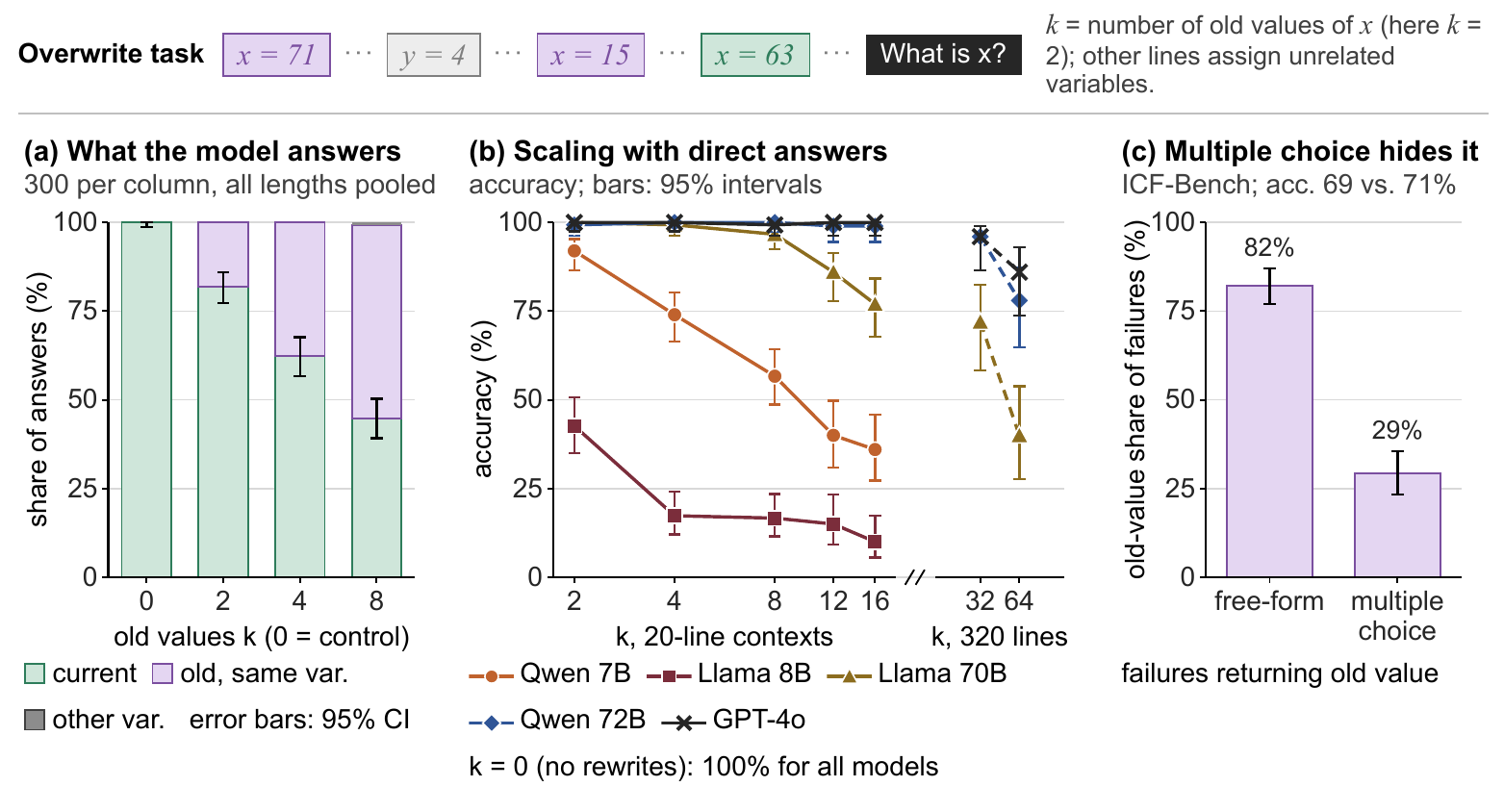}
\caption{Old values, not length alone. \textbf{a}, Repeated updates mainly
produce old-value errors; longer contexts amplify the effect.
\textbf{b}, Old-value errors persist at $k=64$ with direct answers.
\textbf{c}, Multiple choice masks old-value retrieval on ICF-Bench at similar
accuracy. Bars: 95\% intervals; sample sizes and estimation in
Appendices~\ref{app:synthetic-controls} and~\ref{sec:icf-scoring}.}
  \label{fig:phenomenon}
\end{figure}
Spending more computation at answer time does not remove it either. On paired
80-line prompts, a larger direct-answer budget alone leaves Qwen's overwrite
accuracy unchanged, whereas concise reasoning improves it; every remaining error
still returns an old value (Appendix~\ref{app:reasoning-results}). That benefit
does not survive load: at $k=256$, GPT-4o returns old values on 32 of 36 prompts
even with concise reasoning, although GPT-5 and Gemini answer all 36 correctly
(Appendix Table~\ref{tab:highload-compact}).

\paragraph{Frontier failures under complex updates.}
The error survives into frontier models once the update structure grows complex.
In 2,048-operation warehouse and build-release logs, with complete histories and
explicit rules, GPT-5.6 Sol answers 9/40 current-state queries correctly and
Opus 4.8 18/40 (18 of 31 completed responses; nine more reach the output cap).
Both answer all 40 paired snapshots correctly, so the difficulty is
reconstructing the state through the history, not reading it. In an exploratory,
pre-specified 17-scenario prefix whose operations take effect only when a later
line confirms them, Opus scores 12/17, against 16/17 on a control with the same
operations at the same line numbers but each effect stated where it happens.
Appendix~\ref{app:frontier-state-logs} gives settings, paired outcomes, and
verified examples.

Apart from the scaling results, the answer format decides whether the error is
visible. On ICF-Bench Dynamic Preference, Qwen2.5-7B reaches nearly the same
accuracy in both formats, yet 82.1\% of its free-form failures return an old
preference against 29.4\% under multiple choice
(Figure~\ref{fig:phenomenon}c): supplying options masks old-value retrieval
rather than preventing it. Accuracy alone therefore cannot reveal what a model
does with an update. These tasks need a diagnostic benchmark that labels which
value a response reused, not only whether it was correct. Motivated by this observation,  in 
Section~\ref{sec:cicm} we will introduce the diagnostic benchmark that we have built for measuring LLMs' in-context update behavior. Appendix Table~\ref{tab:discrete-summary-full} and
Appendix~\ref{sec:icf-scoring} give the comparison and its scoring.

\section{CICM: A Benchmark for In-Context Updates}
\label{sec:cicm}
\label{sec:formulation}
In this section, we propose CICM as a comprehensive and diagnostic benchmark to test whether models use the current state while earlier versions remain
in context. In this benchmark, by setting explicit update rules and programmatic ground truth, we can precisely diagnose
which information an incorrect answer reuses (Figure~\ref{fig:cicm-examples}). CICM's core components can be summarized as follows.

\paragraph{Dialogue updates.}
The core uses PrefEval content \citep{zhao2025prefeval} in 1,200 dialogues about
meal, music, and learning preferences. A generator controls values, update order,
and competing mentions; language models provide wording. Annotated positions
support probes and interventions. Another 180 pilot ledgers span six domains,
with single-value, partial-record, and full-record updates.

\paragraph{Decisions and operational state.}
In 24 procurement and scheduling scenarios, revised constraints determine a
unique optimal action. Forty primary operational histories require tracking
slot contents through 2,048 dependent handoffs over 16 slots. A deferred variant
adds operations that execute only on later confirmation. Paired snapshots
supply the resolved state while preserving the query.

\paragraph{Construction and scoring.}
Answers are fixed before evaluation and checked through span validation,
independent state replay, or action enumeration. Direct-value scoring distinguishes
current, old, other-variable, and other answers; decision scoring measures
optimality and constraint satisfaction. Response availability is recorded
separately. Appendix~\ref{app:cicm-datasheet} details construction and provenance;
Appendix~\ref{app:test-time-repair} gives mixed-answer scoring checks.

\begin{figure}[b]
\centering
\begingroup
\tcbset{cicmexample/.style={enhanced,breakable=false,colback=white,
 colframe=neutralink,boxrule=0.5pt,arc=1mm,left=2mm,right=2mm,
 top=1.5mm,bottom=1.5mm,fontupper=\footnotesize,
 fonttitle=\bfseries\small,colbacktitle=neutralfill,coltitle=black,
 equal height group=cicm-main,before skip=0pt,after skip=0pt}}
\begin{minipage}[t]{0.49\linewidth}
\begin{tcolorbox}[cicmexample,title={(a) Updating a preference}]
\textbf{Dialogue excerpts}\par\smallskip
``I think \textcolor{staleink}{\textbf{visual}} works best for me overall.''\par\smallskip
``I actually prefer \textcolor{currentink}{\textbf{reading}} when I'm trying to understand something new.''\par\smallskip
\textcolor{neutralink}{\emph{[Other turns and preference updates omitted.]}}\par\smallskip
``For the record, \textcolor{staleink}{\textbf{visual}} was an earlier choice for my learning style, and is no longer current.''\par\medskip
\textbf{Query (summary)}\par
What is my latest explicitly updated learning style? Later mentions are not updates.\par\medskip
\textbf{Qwen2.5-7B:} \textcolor{staleink}{\textbf{visual}}\par
\textbf{Current answer:} \textcolor{currentink}{\textbf{reading}}
\end{tcolorbox}
\end{minipage}\hfill
\begin{minipage}[t]{0.49\linewidth}
\begin{tcolorbox}[cicmexample,title={(b) Tracking a deferred handoff}]
\textbf{Agent-log excerpts}\par\smallskip
\texttt{0876 handoff(}\par
\texttt{\quad cycle=[S08,S09,S14,S12]);}\par
\texttt{branch=production; status=pending}\par\smallskip
\textcolor{neutralink}{\emph{[Intervening log entries omitted.]}}\par\smallskip
\texttt{1007 resolve(seq=0876);}\par
\texttt{result=committed}\par\smallskip
\textcolor{neutralink}{\emph{[Remaining log entries omitted.]}}\par\medskip
\textbf{Query (summary)}\par
Which parcel is currently in S14?\par\medskip
\textbf{Opus 4.8:} \textcolor{staleink}{\textbf{P7888}}\par
\textbf{Current answer:} \textcolor{currentink}{\textbf{P7238}}
\end{tcolorbox}
\end{minipage}
\par\vspace{2mm}
\begin{tcolorbox}[enhanced,breakable=false,colback=currentfill,colframe=currentink,
boxrule=0.5pt,arc=1mm,left=2mm,right=2mm,top=1mm,bottom=1mm,
fontupper=\footnotesize,before skip=0pt,after skip=0pt]
\textbf{Programmatic state check for (b).} A pending handoff executes only when
committed by a later resolve line. At line 1007, S09 holds P7238; executing
the cycle moves it to S14. No later executed operation changes S14.
This annotation explains the answer; it was not supplied to the model.
\end{tcolorbox}
\endgroup
\caption{CICM examples.
Full cases: Appendices~\ref{app:failure-examples} and~\ref{app:frontier-state-logs}.}
\label{fig:cicm-examples}
\end{figure}

\subsection{Dialogue validity and decision outcomes}
\label{sec:cicm-validity}

The original preference core adds an ambiguous old-value re-mention in its
near condition. Paired follow-ups clarify wording and isolate placement:
with identical messages and mention counts, moving an explicitly historical
mention nearer the query raises Qwen's old-value rate from 14.8\% to 36.7\%;
GPT-4o returns no old values in either condition
(Appendix~\ref{app:cicm-validity-results}). On the decision component, four
reasoning models return optimal actions on 88--95 of 96 full-history prompts;
remaining outcomes are invalid outputs, provider failures, or truncations
(Appendix~\ref{app:decision-transfer}).

\section{Mechanism Analysis}
\label{sec:mechanism}

We ask what the model selects, whether the current value is still available,
where attention goes, which components decide, and whether selection can be
overridden. Attention drift appears in five models under matched controls, and
the tests that open the model up run in Pythia-160M, Qwen2.5-7B and
Llama-3.1-8B. Appendix Table~\ref{tab:mechanism-evidence-map} records which
model supports each claim.

\subsection{Identity and recency compete}
Identity outweighs recency: an old value of the queried variable dominates
responses even when another variable was updated more recently. We test this on
CICM's preference dialogues with Qwen2.5-7B, querying the current value of one
variable. The factorial holds the current assignment far from the query and
crosses a later re-mention of the most recent old value (present or absent) with
another variable's distance (far, intermediate, or two turns), 200 dialogues per
condition (Figure~\ref{fig:factorial}).

\begin{figure}[tb]
  \centering
  \includegraphics[width=\textwidth]{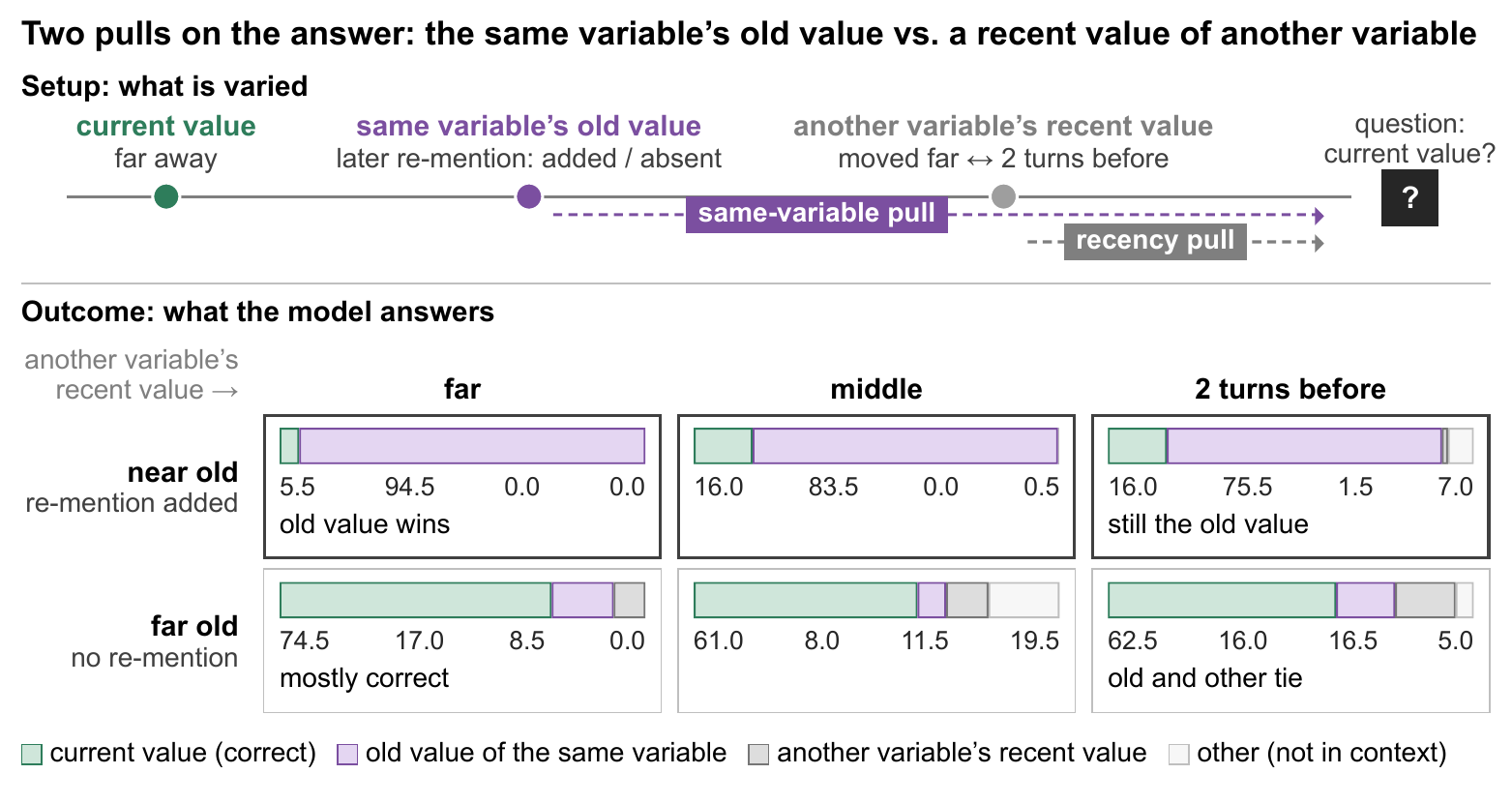}
  \caption{Competing mentions in CICM.
Numbers are response percentages.}
  \label{fig:factorial}
\end{figure}

A nearby re-mention of the old value decides the answer. When the old value is
re-mentioned near the query, old-value errors dominate (75.5--94.5\%) even
though another variable is more recent. Without that re-mention, bringing the
other variable close makes old-value and other-variable errors about equally
frequent (16.0\% and 16.5\%). Recency competes only when the old value is not
repeated nearby. The re-mention also decides which old value comes back. For $k=2,3,4,6$, the
re-mentioned value accounts for 94.5--100\% of old-value answers in both Qwen
and Llama. Without it, the most recent old value accounts for 66.7--100\% in
Qwen and 44.2--61.1\% in Llama (Appendix~\ref{app:which-old}). A re-mention can
be read as a renewed preference rather than a repetition. Marking it explicitly
historical lowers Qwen's near-source old-value rate to 25.5\% without
eliminating the errors (Appendix~\ref{app:cicm-validity-results}). The original
contrast therefore combines wording and mention count with position, and the
matched-placement test in Section~\ref{sec:cicm-validity} isolates position.

\subsection{The current value remains readable}
When a model returns an old value, one possible explanation is that it
has lost the current value from its internal representation.
However, a linear probe can still recover information about the current
value from the model's hidden state, even when its answer is incorrect. On old-value errors, a linear probe
assigns the current value a mean probability of 84.8\% in Qwen and 82.2\% in
Llama, against 3.2--6.4\% when the value labels are randomly reassigned
(Figure~\ref{fig:qk-attention}a). Other-variable errors show similar
recoverability (Appendix Table~\ref{tab:retention-attention}). The probe reads the standardized final-layer hidden state $h_i$ at the answer
position in dialogue $i$. With learned weights $W$ and bias $b$, denoted
$\phi=(W,b)$, it assigns each of the 21 candidate values $v$ the probability
$q_\phi(v\mid h_i)=\operatorname{softmax}(Wh_i+b)_v$. We fit and evaluate it in
five folds split by dialogue, so each reported score comes from a probe that did
not train on that dialogue. Randomly reassigning the value labels checks whether
the scores depend on the true relationship between hidden states and values
\citep{hewitt2019probes}. These are probe probabilities, not answer accuracies.

Readability does fall on failures. After accounting for prompt length, scores on
old-value errors are about 7\% below those on correct
answers in both models. The current value therefore stays accessible to a linear
classifier without being used, which is not yet evidence that the model consults
it. The next tests examine that selection step, and
Appendix~\ref{app:probe-method} gives the probe details.

\subsection{Old values take priority in attention}

Attention drift is not specific to one model. We use \emph{attention drift} for
attention giving greater priority to old values than to the current value.
Comparing each updated dialogue with one of equal token length but no
conflicting reassignment, conflicting updates shift attention and query--key
matching toward old values in five models with reliable controls, including
Llama and answers that are correct (Appendix~\ref{app:stage-p}). Redirecting
that attention changes which value is answered in Qwen, Llama and Mistral
(Section~\ref{sec:test-time-repair}), so the shift is not incidental to the
answer. In Qwen, the same measure separates failures from correct answers. For each head
at the answer position, let $A_{\mathrm{old}}$ and $A_{\mathrm{current}}$ be the
attention masses summed over old and current assignments, and define the
old-value share as
\begin{equation}
\rho_{\mathrm{old}}=
\frac{A_{\mathrm{old}}}{A_{\mathrm{old}}+A_{\mathrm{current}}}.
\label{eq:old-attention-share}
\end{equation}
Averaged across heads and adjusted for prompt length, the share is 19.4
percentage points higher on failures, with the largest shift in late layers.
The matched-context shift and the failure-associated shift measure different
contrasts.

\begin{figure}[tb]
  \centering
  \includegraphics[width=\textwidth]{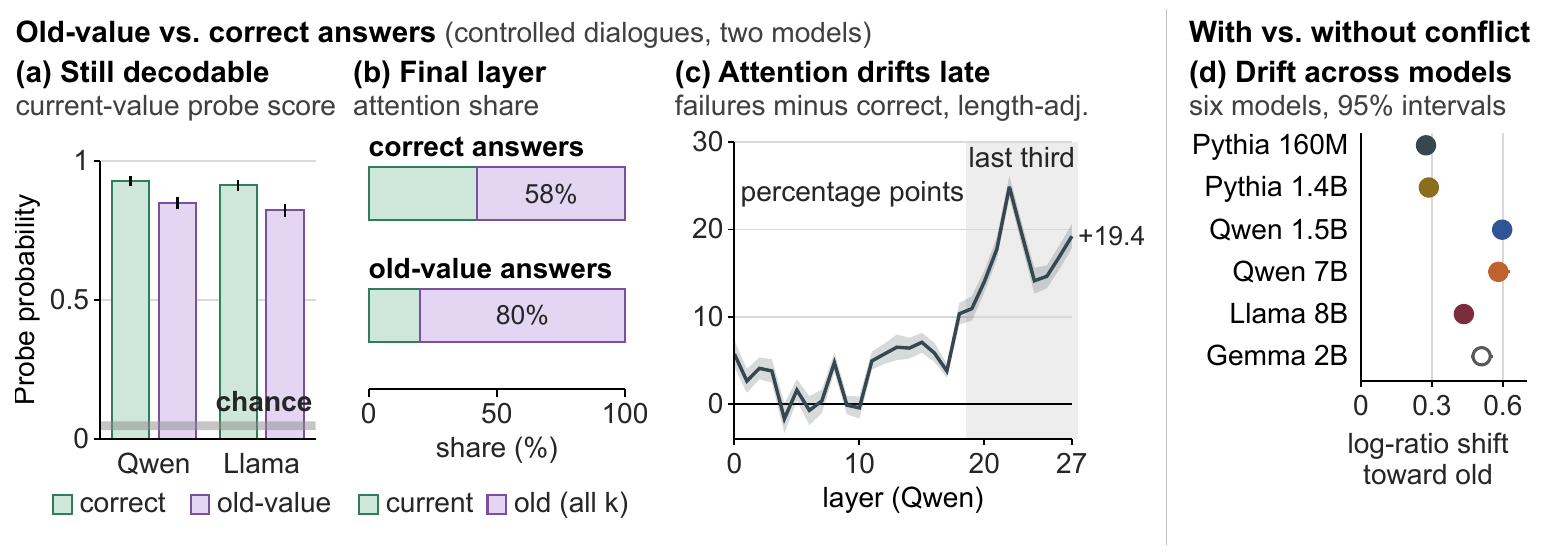}
  \caption{Attention drift: old values take priority.
\textbf{a}, Mean probe probability assigned to the current value
(grey: shuffled labels). \textbf{b}, Qwen's final-layer attention favors old values on failures.
\textbf{c}, In Qwen the length-adjusted failure--correct difference grows in
late layers. \textbf{d}, Matched conflict--control shifts.}
\label{fig:qk-attention}
\end{figure}

\subsection{Which components decide the answer}

A shift in attention on failures is still a correlation. By replacing components we can causally test whether they produce the answer.

Attention selects positions through query--key matching and transmits their
information through value vectors and the output projection
\citep{vaswani2017attention}. For 120 failed synthetic overwrite prompts, we
construct successful counterparts with identical candidate values and positions,
renaming earlier assignments to remove competition for the queried variable
(Appendix~\ref{app:component-replacements}). We then replace components in
failed runs with their successful counterparts (Qwen rows in
Table~\ref{tab:mechanism-recovery}). Recovery concentrates in late hidden states
and old-value keys, implicating late selection and competing keys. For the
current-minus-old answer-score gap $g$, we measure recovery after replacement as
\begin{equation}
R=100\,\frac{g_{\mathrm{replaced}}-g_{\mathrm{failed}}}
{g_{\mathrm{successful}}-g_{\mathrm{failed}}}.
\label{eq:component-recovery}
\end{equation}
Here $R=100$ restores the successful gap, and $R<0$ favors the old value further.

\begin{wraptable}{r}{0.51\textwidth}
\centering
\small
\setlength{\tabcolsep}{3pt}
\caption{Causal evidence in Qwen-7B and Pythia-160M. Each row specifies its comparison.}
\label{tab:mechanism-recovery}
\begin{tabularx}{\linewidth}{@{}Xr@{}}
\toprule
Intervention comparison & Result \\
\midrule
\multicolumn{2}{@{}l}{\textit{Answer-score-gap recovery (\%)}} \\
Qwen: late / early--middle states & $78.1\;/\;-0.5$ \\
Qwen: old / current keys & $49.2\;/\;-21.3$ \\
Pythia key inputs: targeted / random & $75\;/\;6$ \\
\addlinespace
\multicolumn{2}{@{}l}{\textit{Old-value errors corrected (\%)}} \\
Pythia head removal: targeted / random & $38.4\;/\;8.7$ \\
\bottomrule
\end{tabularx}
\par\smallskip
{\footnotesize\raggedright Pythia comparisons are held out, and random sets are matched by layer.
Full controls: Appendices~\ref{app:component-replacements} and~\ref{app:stage-n-details}.\par}
\end{wraptable}
Replacement narrows the error to a pathway rather than to individual heads.
Naming them takes exhaustive intervention: removing each head in turn, replacing
the inputs to its key and query, and matching every result against
layer-matched random controls. Pythia-160M produces the same old-value answers
and is small enough to carry that program through. Removing the heads selected
to promote old values corrects 38.4\% of held-out old-value answers, against
8.7\% for layer-matched random sets, with discovery and evaluation on separate
examples. The recovery concentrates in the query--key path, where matching
favors old assignments while the value pathway keeps transmitting whatever is
attended. Replacing selected inputs to a single head's key recovers 75\% of the
answer-score gap, against 6\% for layer-matched random inputs
(Table~\ref{tab:mechanism-recovery}). As in the larger models, the current value
remains readable on these failures. Appendix~\ref{app:stage-n-details} gives the
full experimental results.

\subsection{Selection is decided late}

Locating the components is not the same as being able to change what they do,
and when the answer is settled matters for both. Repairing the update once does
not secure it. In Pythia-160M, replacing the update's attention key with the
counterpart from a successful run corrects 92.8\% of failures when the question
follows the update immediately, and 8.7\% of the same failures when the
remaining turns, including a later mention of the old value, run first.
Selection must therefore be maintained through the context that follows, not
only established at the update (Appendix~\ref{app:stage-p}).

\textbf{A late change does work}. Adding the difference between the probe weights for the
current and competing values at the answer position can favor the current value
even when the model initially answers incorrectly, and we compare this with an
equally large random change (Equation~\ref{eq:steer}). Its advantage grows with
strength at the final layer, and the same test near the middle has much smaller
effects. At relative strength $\alpha=0.4$, targeted steering reduces
persistence of the original error category by 39.5\% more than
random steering in Qwen and 69.5\% in Llama (Appendix
Figure~\ref{fig:causality}). Because it uses the known current value at the
final state, this is a diagnostic intervention: it can directly favor the answer
without locating the original error. Selection stays open to change at the end
of the computation, which is where Section~\ref{sec:test-time-repair}
intervenes.

\section{A theoretical explanation of old-value selection}
\label{sec:theory}

Our experiments show that models can return an old value even when
the current value remains recoverable from their hidden states.
We provide a theoretical explanation using a one-layer Transformer
with a simplified copying mechanism: both values remain available,
and the probability of copying each value equals the attention weight
on its position. This setting lets us examine how selection can fail
without information being lost.

Consider a current value $v_c$ and an old value $v_o$, with attention
scores $s_c,s_o$ and distances $d_c<d_o$ from the query.
Their selection probabilities $P(\cdot)$ depend on the difference between
their attention scores. For a shared content key with rotary position
encoding \citep{su2024roformer}, this relationship becomes
\begin{equation}
 \log\frac{P(v_c)}{P(v_o)}
 =\frac{s_c-s_o}{T}
 =\frac{2}{T}\sum_b w_b\sin(m\theta_b+\psi_b)
   \sin\!\left(\frac{\delta\theta_b}{2}\right),
 \label{eq:main_pairwise}
\end{equation}
where $T>0$ is the softmax temperature,
$\delta=d_o-d_c$, $m=(d_c+d_o)/2$, and
$w_b,\theta_b,\psi_b$ denote the amplitude, frequency, and phase
of rotary band $b$.

Equation~\ref{eq:main_pairwise} shows why retaining the current value
does not guarantee that it will be selected. When the attention-score
difference $s_c-s_o$ is small relative to $T$, $P(v_c)/P(v_o)$ is close to one, meaning an outdated value is nearly as likely to be selected as the current one.Moreover, because the right-hand side of Equation~\ref{eq:main_pairwise} is not strictly positive, a more recent positional placement alone cannot guarantee a higher attention score \citep{du2026rope}. This dynamic provides a clear mathematical explanation for why stale-binding errors occur even when the current information remains perfectly recoverable in our probing experiments. (Detailed derivations are provided in Appendix~\ref{app:theory-core}, and we offer a measured pairwise example in Appendix~\ref{app:pairwise-example}).

The presence of multiple old values further exacerbates this competition. Given $k$ old values, the ratio of the probability of selecting
the current value to that of selecting any old value is
$e^{s_c/T}/\sum_{j=1}^{k}e^{s_j/T}$.
If we hold existing scores constant, injecting an additional old value strictly decreases this ratio, despite the current value's score remaining entirely unchanged. This accumulation illustrates how successive updates mathematically dilute the likelihood of selecting the current state without ever erasing the underlying information. The central implication is that information retention and attention-based selection are fundamentally distinct requirements: accumulated old values remain formidable competitors for retrieval even when the correct current state is fully available.
\section{Controlling selection at test time}
\label{sec:test-time-repair}

We test whether redirecting attention from old values to the current value
can correct the model's answer without changing its weights
\citep{zhang2024pasta}. This intervention improves current-value retrieval
on controlled preference dialogues (Table~\ref{tab:test-time-repair}).
For comparison, we also append the current value identified by the same
parser as an explicit reminder. 
At inference time, a rule-based parser uses the dialogue's predefined
variables, values, and update patterns to identify the current value
and its earlier alternatives. Because the parser already identifies
the correct value, this experiment tests whether redirecting attention
can make the model use it. We adjust selected attention heads at the final prompt position.
Let $C$ and $O$ contain the token positions of the queried variable's
current and old values, respectively. For each selected head, we add
a bias $\beta \geq 0$ to the attention scores at current-value positions
and subtract it at old-value positions:
\begin{equation}
s'_j = s_j + \beta\,\mathbf{1}[j \in C]
           - \beta\,\mathbf{1}[j \in O],
\qquad
a_j = \frac{e^{s_j}}{\sum_{\ell} e^{s_{\ell}}},
\quad
a'_j = \frac{e^{s'_j}}{\sum_{\ell} e^{s'_{\ell}}}.
\label{eq:attention-routing}
\end{equation}
Here $s_j$ denotes the original attention score, while $a_j$ and $a'_j$
are the attention weights before and after routing, normalized over
all prompt positions. The bias $\beta$ shifts attention toward the
current value and away from old values. We consider two forms of routing: a fixed bias calibrated in advance,
and an adaptive bias that adjusts to each input to reach a target
current-to-old attention ratio, subject to a cap. Heads and routing
settings are selected before testing, without using test labels
(Appendix~\ref{app:test-time-repair}).
\begin{table}[htbp]
\centering
\small
\setlength{\tabcolsep}{5pt}
\caption{ Accuracy and preservation are percentages; old-error correction is shown as counts.}
\label{tab:test-time-repair}
\begin{tabular}{@{}lrrrrr@{}}
\toprule
& \multicolumn{3}{c}{Accuracy $\uparrow$} & Old errors & Preserved \\
\cmidrule(lr){2-4}
Model & Original & Routing & Reminder & corrected $\uparrow$ & $\uparrow$ \\
\midrule
\multicolumn{6}{@{}l}{\textit{Adaptive attention bias}} \\
\rowcolor{neutralfill} Qwen2.5-3B & 16.46 & 95.21 & \textbf{99.27} & 678/711 & 98.73 \\
\rowcolor{neutralfill} Llama-3.2-3B & 27.71 & \textbf{93.12} & 68.85 & 514/541 & 100.00 \\
\rowcolor{neutralfill} Llama-3.1-8B & 41.71 & \textbf{88.63} & 65.38 & 395/467 & 100.00 \\
\rowcolor{neutralfill} Mistral-7B & 5.42 & 81.35 & \textbf{98.44} & 564/674 & 98.08 \\
\rowcolor{neutralfill} Gemma-2-9B & 47.08 & \textbf{91.88} & 82.81 & 407/455 & 100.00 \\
\addlinespace[2pt]
\multicolumn{6}{@{}l}{\textit{Fixed attention bias}} \\
Qwen2.5-7B & 37.60 & 68.85 & \textbf{88.02} & 208/487 & 100.00 \\
Qwen2.5-14B & 40.52 & 41.67 & \textbf{98.54} & 13/491 & 98.46 \\
\bottomrule
\end{tabular}
\par\smallskip
\parbox{\linewidth}{\footnotesize Reminder adds the parsed current value before the query.
Old errors corrected shows corrected/initial old-value errors;
Preserved shows the percentage of initially correct answers retained.}
\end{table}

Adaptive routing improves current-value retrieval while largely preserving
answers that were already correct. Across the five adaptive configurations,
absolute accuracy gains range from 44.79 to 78.75\%, with 98.08 to 100\% of
initially correct answers preserved (Table~\ref{tab:test-time-repair}).
These improvements depend on which positions and heads are targeted:
all five gains exceed the observed ranges of random-head and
random-position controls, while reversing the routing reduces accuracy
(Appendix Table~\ref{tab:repair-controls}). Routing also outperforms an explicit current-value reminder in both
Llama models and Gemma, although reminders remain stronger in Qwen
and Mistral. Together, these results show that directing attention
toward the current value can substantially improve answer selection.

\section{Related Work}

\paragraph{Updating information in context.}
Memory benchmarks test knowledge and preference updates in conversation
\citep{wu2025longmemeval,zhao2025prefeval,jiang2025personamem} and agent-maintained
memory \citep{cheng2026amemgym,uddin2026memora,xie2026dynamicmem,patel2026supersede}.
Earlier dialogue systems remove invalidated memories \citep{bae2022keep}.
In-context editing \citep{zheng2023ike} and learned context-dependent unlearning
\citep{takashiro2025unlearning} control factual responses. ICF-Bench evaluates
superseded instructions and preferences \citep{icfbench2026}.

\paragraph{Interference and knowledge conflicts.}
PI-LLM establishes declining recall under repeated key--value updates
\citep{wang2026unable}.
Our overwrite errors favor recent old assignments, where
\citet{chattaraj2026dual} report primacy intrusion under a different protocol.
Other work studies cache suppression and quantization
\citep{xie2026sleepgate,shahrabi2026compress}.
Position effects are studied by \citet{liu2024lost,hsieh2024found}.
Knowledge-conflict studies contrast context with parametric memory
\citep{xu2024knowledgeconflicts,xie2024chameleon}, while context-aware decoding
amplifies contextual evidence \citep{shi2024cad}.

\paragraph{Binding and selection mechanisms.}
Entity-tracking tasks \citep{kim-schuster-2023-entity,li2025statetracking,wu2025variablebinding}
and circuit analyses \citep{prakash2024finetuning,feng2024binding,gurarieh2026mixing}
examine state and binding representations. \citet{tang2026entitystate} find
parallel aggregation of state-changing operations at the query position.
We instead ask why an updated value remains recoverable when a superseded
value is selected. \citet{oh2026rebinding} identify retrieval-conditioned
rebinding circuits. Retrieval heads identify copying routes
\citep{wu2025retrievalheads}. We study competition among successive
assignments of the same variable. Our Qwen comparison gives distinct, partly
overlapping rankings for single-write copying and attention to the latest
assignment (Appendix~\ref{app:retrieval-head-comparison}).

\paragraph{Test-time intervention.}
Activation steering changes outputs through learned directions
\citep{turner2023activation,li2023iti}. PASTA reweights selected attention heads
toward supplied emphasis spans \citep{zhang2024pasta}. Our routing derives
current and superseded spans from dialogue updates, then calibrates heads
and intervention strength to redirect selection toward the current value.
This tests causal control on direct questions, without changing model weights.

\section{Limitations}

Our controlled dialogues and constructed operational logs do not estimate
failure prevalence in deployed agents. Internal measurements only cover open-source  models,
with different patterns across comparisons. Probes establish
recoverability, while known-answer steering establishes output control.  Routing requires
internal access, labeled calibration, and a structured parser.

\section{Conclusion}

Reliable context management requires both retaining updates and using the
current state. CICM brings dialogue updates, revised-constraint decisions, and
complex operational histories into a common framework for diagnosing update
failures. Our open-model experiments show that the current value can remain
recoverable even when the answer selects an old one.
Probes reveal a selection bottleneck, and component interventions in Qwen and
Pythia identify how attention drift toward old values contributes to it.
A simple one-layer model explains how old values can collectively receive more
attention than the current value. Training-free attention routing improves
direct answers, correcting most old-value errors with input-adjusted
interventions across Qwen, Llama, Mistral, and Gemma. The challenge also reaches frontier reasoning models: GPT-5.6 Sol and Claude
Opus 4.8 produce incorrect current-state answers on complex histories even
when their reasoning completes.
Access to the full history and extensive reasoning therefore do not guarantee correct use of an evolving state. Together, our benchmark, mechanism analysis, and interventions make reliable update selection a concrete target for
building and evaluating LLM agents.

\FloatBarrier
\section*{Acknowledgement}
This work was supported by the U. S. Army Research Laboratory and
the U. S. Army Research Office under Grant W911NF2010219, Office of
Naval Research under Grant N000142412673, and NSF.

\section*{Code and Data Availability}

The experiment code and CICM benchmark dataset are publicly available at
\url{https://github.com/JamesJunyuGuo/CICM}.
The appendices describe the task construction, evaluation protocols,
model settings, and statistical analyses.
\begingroup
\setlength{\bibsep}{3pt plus 0.5pt minus 0.5pt}
\bibliography{references}
\bibliographystyle{iclr2027_conference}
\endgroup

\appendix
\let\topfraction\cicmSavedTopFraction
\let\bottomfraction\cicmSavedBottomFraction
\let\textfraction\cicmSavedTextFraction
\setlength{\textfloatsep}{\cicmSavedTextFloatSep}
\setcounter{topnumber}{3}
\setcounter{bottomnumber}{2}
\setcounter{totalnumber}{5}
\makeatletter
\setlength{\@fptop}{0pt}
\setlength{\@fpsep}{18pt plus 2pt minus 2pt}
\makeatother
\section*{Appendix guide}
The appendices provide derivations, additional evidence, and replication
details. For the mechanism results, start with
Appendices~\ref{app:retention-attention}--\ref{app:stage-p}.

\begingroup
\small
\renewcommand{\arraystretch}{1.08}
\noindent
\begin{tabularx}{\linewidth}{@{}lXr@{}}
 & & Page \\
\ref{app:theory} & \hyperref[app:theory]{Theory: assumptions and proofs} & \pageref{app:theory} \\
\ref{app:extended} & \hyperref[app:extended]{Additional results} & \pageref{app:extended} \\
\ref{app:failure-examples} & \hyperref[app:failure-examples]{Failure examples} & \pageref{app:failure-examples} \\
\ref{app:which-old} & \hyperref[app:which-old]{Which old value is returned?} & \pageref{app:which-old} \\
\ref{sec:community-benchmarks} & \hyperref[sec:community-benchmarks]{Independent community benchmarks} & \pageref{sec:community-benchmarks} \\
\ref{sec:icf-scoring} & \hyperref[sec:icf-scoring]{ICF-Bench scoring and answer format} & \pageref{sec:icf-scoring} \\
\ref{app:retention-attention} & \hyperref[app:retention-attention]{Retained information and attention} & \pageref{app:retention-attention} \\
\ref{app:component-replacements} & \hyperref[app:component-replacements]{Component replacements} & \pageref{app:component-replacements} \\
\ref{app:steering} & \hyperref[app:steering]{Steering at the answer position} & \pageref{app:steering} \\
\ref{app:stage-n-details} & \hyperref[app:stage-n-details]{Tracing attention in Pythia-160M} & \pageref{app:stage-n-details} \\
\ref{app:geometry} & \hyperref[app:geometry]{Testing a linear account of identity and recency} & \pageref{app:geometry} \\
\ref{app:stage-p} & \hyperref[app:stage-p]{When old-value answers develop} & \pageref{app:stage-p} \\
\ref{app:experimental-setup} & \hyperref[app:experimental-setup]{Experimental setup and reproducibility} & \pageref{app:experimental-setup} \\
\ref{app:cicm-datasheet} & \hyperref[app:cicm-datasheet]{CICM datasheet} & \pageref{app:cicm-datasheet} \\
\ref{app:test-time-repair} & \hyperref[app:test-time-repair]{Test-time attention routing: method, controls, and scope} & \pageref{app:test-time-repair} \\
\ref{app:api-robustness} & \hyperref[app:api-robustness]{Behavioral robustness and task diversity} & \pageref{app:api-robustness} \\
\ref{app:cicm-diversity} & \hyperref[app:cicm-diversity]{Task diversity and blinded semantic review} & \pageref{app:cicm-diversity} \\
\ref{app:decision-transfer} & \hyperref[app:decision-transfer]{Decisions after scoped updates in dialogues and agent logs} & \pageref{app:decision-transfer} \\
\ref{app:frontier-state-logs} & \hyperref[app:frontier-state-logs]{Frontier failures on complex operational histories} & \pageref{app:frontier-state-logs} \\
\end{tabularx}\par
\endgroup

\section{Theory: Assumptions and Proofs}
\label{app:theory}

We give the assumptions and proofs for the one-layer model in
Section~\ref{sec:theory}, covering pairwise positional matching and competition
among retained assignments.

\subsection{One-layer attention model}
\label{app:theory-core}

Consider the assignments associated with a single queried variable. Let
assignment $i$ at
distance $d_i$ from the answer query $q$ receive score
$s_i=\langle q,R_{-d_i}k_0\rangle/\sqrt{d_h}$. We assume these assignments
have the same content key $k_0$, so their scores differ only through position.
Here $R_{-d_i}$ is the RoPE rotation for relative distance $d_i$, and $d_h$
is the attention-head dimension. Let $T>0$ be the softmax temperature. The attention weight is
$a_i=e^{s_i/T}/\sum_j e^{s_j/T}$. We assume that each assignment carries a distinct
answer $v_i$, that attending to an assignment copies its value faithfully, and that
later computation does not change the copied identity. Let $Y$ denote the
returned value. Its distribution is then
\begin{equation}
 \Pr(Y=v_i)=a_i.
 \label{eq:copy}
\end{equation}
Below, $P(v_i)$ abbreviates $\Pr(Y=v_i)$ within this model. Let  $b$ be the
two-dimensional rotary band, $q_b$ and $k_{0,b}$ be the
corresponding query and key blocks, $\theta_b$ be the band's angular frequency,
$\psi_b$ be the angle between these blocks, and
$w_b=\lVert q_b\rVert\lVert k_{0,b}\rVert/\sqrt{d_h}$ be their
content-dependent amplitude. The score of the assignment at distance $d$ from the query is then
\begin{equation}
 \varphi(d)=\langle q,R_{-d}k_0\rangle/\sqrt{d_h}
 =\sum_b w_b\cos(d\theta_b+\psi_b),
 \label{eq:rope_score}
\end{equation}

\begin{proposition}[Pairwise odds identity]\label{prop:oddsid}
For a current assignment $c$ at distance $d_c$ and an old assignment $o$ at
distance $d_o>d_c$, let $\delta=d_o-d_c$ and
$m=\tfrac12(d_c+d_o)$. Then
\begin{equation}
 \log\frac{P(v_c)}{P(v_o)}
 =\frac{s_c-s_o}{T}
 =\frac{2}{T}\sum_b w_b\sin(m\theta_b+\psi_b)
   \sin\!\Big(\tfrac{\delta\theta_b}{2}\Big).
 \label{eq:odds_id}
\end{equation}
\end{proposition}
\begin{proof}
Equation~\ref{eq:copy} gives
$P(v_c)/P(v_o)=a_c/a_o=e^{(s_c-s_o)/T}$ because the softmax normalizer
cancels. Applying
$\cos A-\cos B=-2\sin\tfrac{A+B}{2}\sin\tfrac{A-B}{2}$ to
$\varphi(d_c)-\varphi(d_o)$ gives the final expression.
\end{proof}

\begin{corollary}[Pairwise rotary-margin bound]
\label{cor:pairwise-bound}
Define
\begin{equation}
 B(\delta)=\frac{2}{T}\sum_b w_b
 \left|\sin\!\left(\frac{\delta\theta_b}{2}\right)\right|.
 \label{eq:pairwise_bound_B}
\end{equation}
Then we have
\begin{equation}
 e^{-B(\delta)}\leq\frac{P(v_c)}{P(v_o)}\leq e^{B(\delta)}.
 \label{eq:pairwise_bound}
\end{equation}
Consequently, when the band-weighted positional difference $B(\delta)$ is
small, the two assignments must have similar selection probabilities.
\end{corollary}
\begin{proof}
Taking absolute values in Equation~\ref{eq:odds_id}, applying the triangle
inequality, and using $|\sin(m\theta_b+\psi_b)|\leq1$ gives
$|\log(P(v_c)/P(v_o))|\leq B(\delta)$. Exponentiating yields the result.
\end{proof}

\begin{proposition}[Single-band translation reversal]\label{prop:reverse}
Fix the content and assignment order, and let the pair have separation $\delta$
and midpoint $m$. Translate both assignments by an integer $\tau$. In the
single-band model, let $w$ be the band's amplitude, define
$A=\tfrac{2w}{T}\sin(\delta\theta_1/2)$, and write
$L(\tau)=A\sin((m+\tau)\theta_1+\psi_1)$. If
$\theta_1\notin2\pi\mathbb Z$ and $L$ is not identically zero, some integer
translations make $L$ positive and others make it negative. Thus the same
ordered pair can favor opposite answers at different absolute placements.
\end{proposition}
\begin{proof}
If $\theta_1/(2\pi)$ is irrational, the integer phase orbit is dense on the
circle and visits both open half-circles. If it equals $r/\ell$ in lowest terms
with $\ell\geq2$, the sum of $L(\tau)$ over one $\ell$-point orbit is zero. Because
the orbit is not identically zero, it contains both signs.
\end{proof}

For example, take a rotary period of 10, $\psi_1=0$, and $w=T=1$. A current
assignment at distance 1 and an old assignment at distance 5 give the current value about
$6{:}1$ odds. Translating both assignments by five positions reverses the scores and
gives the old value the same advantage, although their order is unchanged.
The bound in Corollary~\ref{cor:pairwise-bound} is sufficient rather than
necessary: rotary bands can also cancel one another and produce near-equal
odds when $B(\delta)$ itself is not small. Our empirical QK measurements report
the total score margin, not the individual band terms, so we do not interpret
them as an estimate of $B(\delta)$.

\begin{proposition}[Accumulation of old assignments]\label{prop:groupodds}
Let $C=\{c\}$ contain the current assignment and let
$S_k=\{1,\ldots,k\}$ index the old assignments, whose scores are
$s_1,\ldots,s_k$. Keep these scores and $s_c$ fixed when adding a candidate.
For any group $G$, let $P_k(G)$ denote its total selection
probability under the softmax over $C\cup S_k$, and define
$F_{S,k}=T\log\sum_{j\leq k}e^{s_j/T}$. Adding an old assignment with finite
score $s_{k+1}$ gives
\begin{equation}
 \frac{P_{k+1}(C)}{P_{k+1}(S_{k+1})}
 =\frac{P_k(C)}{P_k(S_k)}
  \frac{1}{1+\exp((s_{k+1}-F_{S,k})/T)}
 <\frac{P_k(C)}{P_k(S_k)}.
 \label{eq:groupodds}
\end{equation}
\end{proposition}
\begin{proof}
Adding $e^{s_{k+1}/T}$ multiplies the old-value partition sum by
$1+e^{(s_{k+1}-F_{S,k})/T}>1$. The current-to-old odds are therefore
multiplied by its reciprocal, which is strictly below one.
\end{proof}

\begin{corollary}[Unrelated candidates leave fixed pairwise odds unchanged]
\label{cor:length}
For any additional candidate set $D$ with fixed scores, let $P'$ denote the
selection distribution after adding $D$. Then
$P'(v_c)/P'(v_o)=e^{(s_c-s_o)/T}=P(v_c)/P(v_o)$.
\end{corollary}
\begin{proof}
The additional candidates change the shared softmax normalizer, which cancels
in the ratio between the two fixed candidates.
\end{proof}

To describe competition between groups, we also consider a softmax over
candidate scores without assuming that they come from a single head. Let $C$, $S$,
and $O$ denote the current-value, old-value, and other-variable candidate
groups. For a group $G\in\{C,S,O\}$ with scores $s_i$ and temperature $T>0$,
write
\begin{equation}
 Z_G=\sum_{i\in G}e^{s_i/T},\qquad
 F_G=T\log Z_G,\qquad
 \pi_G=\frac{e^{F_G/T}}{\sum_H e^{F_H/T}}.
 \label{eq:group_free_energy}
\end{equation}
When restricting the choice to the current- and old-value groups, write
$\pi_{C\mid C\cup S}=Z_C/(Z_C+Z_S)$.
These definitions let us compare groups containing different numbers of
values.

\subsection{RoPE does not guarantee a global recency ordering}

\begin{proposition}[No architecture-guaranteed global recency order under RoPE]
\label{prop:rope}
Fix an attention head with rotary position encoding. Let an answer query sit at
position $n$ and let a queried variable have assignments at positions $p_1<\dots<p_m$
whose keys are identical up to a content perturbation of norm at most $\eta$
(approximately shared content keys). Then the head's
pre-softmax score for assignment $j$ is $\ell_j=g(n-p_j)+O(\eta)$, where
$g(d)=\sum_\omega \alpha_\omega\cos(\omega d+\varphi_\omega)$ is a fixed sum of
cosines over the head's rotary frequency bands, with $\alpha_\omega,\varphi_\omega$
determined by the query and the shared content key. If $g$ is nonconstant, it
is recurrent and cannot be strictly monotone over the unbounded distance
domain. Consequently bare RoPE matching does not guarantee that the nearest
(most recent) assignment receives the largest score for arbitrary spacings. This
global statement does not rule out a particular head being monotone on a
finite realized interval.
\end{proposition}

\begin{proof}
Rotary encoding rotates the query and each key by an angle proportional to
position, so the bilinear score between the query at $n$ and a key at $p_j$
depends on positions only through $n-p_j$: writing the shared content key as
$k_0$, $\langle R_n q, R_{p_j}k_0\rangle=\langle q, R_{-(n-p_j)}k_0\rangle$,
which expands over the rotary frequency bands as
$g(n-p_j)=\sum_\omega\alpha_\omega\cos(\omega(n-p_j)+\varphi_\omega)$. A
content perturbation of norm $\le\eta$ changes the score by $O(\eta)$, giving
$\ell_j=g(n-p_j)+O(\eta)$. If the frequencies are commensurate, nonconstant
$g$ is periodic and therefore cannot be strictly monotone over all distances.
If they are incommensurate, $g$ is almost periodic: it has arbitrarily large
approximate recurrences, which are likewise incompatible with a nonconstant
global monotone order. Thus some distance pairs reverse or tie the desired
nearest-first ordering whenever their positional-score difference dominates
the $O(\eta)$ perturbation. A finite operational interval can still happen to
lie inside a monotone segment, which is why the proposition is a lack of an
architectural guarantee rather than an impossibility claim for every prompt.
\end{proof}

The proposition concerns a head with approximately shared content keys.
A full model can still learn an order-sensitive computation, and rotary encoding alone does not supply the guarantee.

For the full multi-band score, the same calculation yields a weaker but general
conclusion: bare rotary matching cannot maintain a uniformly positive recency
margin across every placement.

\begin{proposition}[No uniform multi-band recency margin]
\label{prop:transrev}
Fix an ordered current--old pair with separation $\delta$ and midpoint $m$,
and translate both assignments by an integer $\tau$. By Proposition~\ref{prop:oddsid}
the log-odds are
$L(\tau)=\tfrac{2}{T}\sum_b w_b\sin((m+\tau)\theta_b+\psi_b)\sin(\delta\theta_b/2)$,
a trigonometric polynomial in $\tau$. If every rotary angle satisfies
$0<\theta_b<2\pi$ (as for RoPE, whose per-position angles lie in $(0,1]$) and
$L$ is not identically zero, then $L$ has zero Ces\`aro mean over integer
$\tau$. Consequently, there is no $\epsilon>0$ such that
$L(\tau)\geq\epsilon$ for every placement $\tau$.
\end{proposition}

\begin{proof}
Expand
$\sin((m+\tau)\theta_b+\psi_b)=\sin(m\theta_b+\psi_b)\cos(\tau\theta_b)+\cos(m\theta_b+\psi_b)\sin(\tau\theta_b)$.
For $0<\theta_b<2\pi$ we have $\theta_b\notin2\pi\mathbb Z$, so
$\big|\sum_{\tau=0}^{N-1}e^{i\tau\theta_b}\big|=\big|(e^{iN\theta_b}-1)/(e^{i\theta_b}-1)\big|\le 2/|e^{i\theta_b}-1|$
is bounded in $N$, giving
$\tfrac1N\sum_{\tau<N}\cos(\tau\theta_b)\to0$ and
$\tfrac1N\sum_{\tau<N}\sin(\tau\theta_b)\to0$. Summing the finitely many bands,
$L$ has zero Ces\`aro mean. A real sequence with zero Ces\`aro mean that is not
identically zero cannot satisfy $L(\tau)\ge\epsilon>0$ for all $\tau$.
\end{proof}

Thus the rotary score alone supplies no placement-independent positive margin
for the current value. This proposition does not require every multi-band head
to reverse its preference on a particular finite context range, but it rules out an architectural guarantee over all placements.

\subsection{How score strength and candidate count combine}

\begin{proposition}[Strength--multiplicity decomposition]
\label{prop:decomposition}
Let $p_i=e^{s_i/T}/Z_G$ be the normalized score within group $G$, and let
$H(p)=-\sum_i p_i\log p_i$ be its entropy. Then
\begin{equation}
 F_G=\sum_{i\in G}p_i s_i+T H(p),
 \qquad
 \max_{i\in G}s_i\leq F_G\leq \max_{i\in G}s_i+T\log |G|,
 \label{eq:entropy_decomposition}
\end{equation}
\end{proposition}

\begin{proof}
From $\log p_i=s_i/T-F_G/T$, averaging over $p_i$ and rearranging gives the
first identity.  The lower bound follows because $Z_G$ contains the largest exponential, and the upper bound follows because every term is at most that
largest exponential.
\end{proof}

\subsection{Selection with bounded score margins}

Proposition~\ref{prop:groupodds} shows how old assignments accumulate
selection mass. Let the current score be $s_c=\mu+\Delta$ and let each of
$k$ old-value scores lie in $[\mu-\varepsilon,\mu+\varepsilon]$.
The following bounds quantify how group competition grows while the current
score remains unchanged.

\begin{proposition}[Multiplicity bound]
\label{prop:multiplicity}
Conditioned on choosing between the current and old-value groups,
\begin{equation}
 \frac{1}{1+k e^{-(\Delta-\varepsilon)/T}}
 \leq \pi_{C\mid C\cup S}\leq
 \frac{1}{1+k e^{-(\Delta+\varepsilon)/T}}.
 \label{eq:robust_overwrite_bound}
\end{equation}
If all old-value scores equal $\mu$ ($\varepsilon=0$), the two bounds coincide.
\end{proposition}

\begin{proof}
Dividing old-value partition mass by current mass gives
$Z_S/Z_C=\sum_{j=1}^{k}e^{(s_j-s_c)/T}$.  Bounding every summand by the two
endpoints yields
$k e^{-(\Delta+\varepsilon)/T}\leq Z_S/Z_C\leq
k e^{-(\Delta-\varepsilon)/T}$, and substitution into
$\pi_{C\mid C\cup S}=1/(1+Z_S/Z_C)$ proves the result.
\end{proof}

\begin{corollary}[Retention and selection can separate]
\label{cor:separation}
For fixed finite $\Delta$ and $\varepsilon$, the current trace and its score
may remain unchanged while $\pi_{C\mid C\cup S}\rightarrow0$ as
$k\rightarrow\infty$.
\end{corollary}

An exact last-write operator could increase $\Delta$ with $k$ and escape the fixed-margin regime, and the corollary says only that preserving a representation
is insufficient when its selection margin does not grow as quickly as the
competing partition mass.

\subsection{Selecting the latest value requires order information}

A second special case makes the same point without any score model. If positions
are uniformly random and the readout is order-blind --- the extreme opposite of a
learned recency term --- the current value is selected only at chance.

\begin{proposition}[Order-invariance excludes latest-binding]
\label{prop:order}
Let the queried variable have $m=k+1$ assignments whose distinct values are assigned to
temporal positions uniformly at random, independent of contents (as in the
controlled generators).  If the score of each assignment depends only on its
content and non-temporal context, then the probability of selecting the latest
assignment, conditioned on selecting a value of the queried variable, is exactly $1/m$.
\end{proposition}

\begin{proof}
Conditioned on contents, an order-blind readout fixes the selection
distribution over assignments, while the uniform ordering makes each assignment
equally likely to be the latest, independently of that distribution. Hence
$\Pr[\text{latest selected}]=\sum_i\Pr[\text{select }i]\cdot\tfrac1m=\tfrac1m$.
\end{proof}

\subsection{Competition between old values and other variables}

Write each effective score as a same-variable identity bonus plus remaining
contextual evidence. If old values of the queried variable share an identity
advantage $b>0$, let $\widetilde F_S$ and $\widetilde F_O$ denote the remaining
contextual contributions to the old-value and other-variable group scores.
Then
\begin{equation}
 F_S=b+\widetilde F_S,\qquad F_O=\widetilde F_O,\qquad
 \log\frac{\pi_S}{\pi_O}
 =\frac{b+\widetilde F_S-\widetilde F_O}{T},
 \label{eq:identity_recency_odds}
\end{equation}
so the boundary between an old-value error and an other-variable error is
\begin{equation}
 \widetilde F_O-\widetilde F_S=b.
 \label{eq:identity_recency_boundary}
\end{equation}
If moving an old value nearer raises $\widetilde F_S$, it strengthens the
old-value group without changing the identity bonus $b$. A nearby value from
another variable can instead raise $\widetilde F_O$ enough to offset that
bonus. This describes the competition in the factorial experiment, but it does not derive a distance-dependent score from RoPE. The candidate counts also
matter through Equation~\ref{eq:entropy_decomposition}.

\subsection{Connection to the mechanism experiments}

The distinction between information availability and selection is central to our theoretical model. A current value can remain perfectly represented in the hidden state even as competing, outdated values accumulate greater attention weight. Diagnostic probes confirm this availability, while our attention and component-replacement experiments isolate the mechanisms of selection. Similarly, modifying final-layer scores can change the output without repairing upstream computations. While our model formally captures these dynamics, it does not attempt to predict the precise, multi-layer computations of a full network.
For example, providing a readout advantage $\delta$ to the current value theoretically reduces old-to-current selection odds by a factor of $e^{-\delta/T}$. This dynamic illustrates how answering formats, such as multiple-choice prompts, might shift the error distribution, though it does not prove that multiple choice operates via a single scalar offset. Likewise, the failed linear prediction of identity and recency (Appendix~\ref{app:geometry}) evaluates a residual-state representation rather than the fixed-head rotary calculation modeled above.

\paragraph{Mechanistic context.}
Our findings complement existing accounts of Transformer retrieval, including induction heads, key–value memories, and factual-association localization \citep{olsson2022induction,geva2021transformer,meng2022locating}. We rely on Pythia checkpoints to analyze the emergence of these computations \citep{biderman2023pythia}, and we utilize shuffled-label controls to validate probe recoverability \citep{hewitt2019probes}. Furthermore, while methods like ROME and MEMIT permanently edit model weights to update facts \citep{meng2022locating,meng2023memit}, our intervention strategy modifies attention directly during inference to observe how dynamic selection shifts the final answer.

\subsection{Connections to human memory research}

We draw a computational analogy to human memory rather than claiming a shared biological mechanism. Tulving's classical distinction between information availability and accessibility provides a foundational framework for this analysis \citep{tulving1966availability}. In proactive-interference experiments, earlier associations directly compete with later ones \citep{underwood1957interference}; similarly, response-competition models distinguish the underlying retention of information from its successful retrieval \citep{mcgeoch1932forgetting}. Phenomena such as cue overload, the fan effect, and the SAM sampling model further describe how multiple stored associations compete during retrieval \citep{watkins1975cueoverload,anderson1974fan,raaijmakers1981sam}. These cognitive parallels motivate our framing of LLM state-tracking errors as a competition among retained associations.

Additionally, temporal-context models and pattern separation describe mechanisms for distinguishing highly similar experiences \citep{howard2002temporal,yassa2011pattern}, while release from proactive interference following a category shift closely parallels our observations regarding variable identity \citep{wickens1970release}. A critical divergence, however, is that an LLM's context window leaves both earlier and updated statements perfectly visible during processing. Because prompt visibility does not guarantee representation within the model's hidden states, our diagnostic probes and attention interventions remain essential. The cognitive correspondence provides a valuable conceptual lens for formulating hypotheses, but it does not replace rigorous mechanistic testing.

\FloatBarrier
\section{Additional Results}
\label{app:extended}

We first give concrete examples and independent behavioral comparisons, then follow
the mechanism from retained information to attention, intervention, and
timing. Table~\ref{tab:mechanism-evidence-map} separates claims by model and comparison, and the following sections provide their evidence.

\begin{table}[htbp]
\centering
\small
\setlength{\tabcolsep}{5pt}
\renewcommand{\arraystretch}{1.15}
\caption{Mechanistic evidence across models.}
\label{tab:mechanism-evidence-map}
\begin{tabularx}{\textwidth}{
  @{}
  >{\raggedright\arraybackslash}p{0.27\textwidth}
  >{\raggedright\arraybackslash}p{0.23\textwidth}
  >{\raggedright\arraybackslash}X
  @{}
}
\toprule
Finding & Models & Key evidence \\
\midrule
Current values remain readable
& Qwen-7B, Llama-8B, Pythia-160M
& Held-out probes recover current values on old-value errors,
above shuffled-label controls. \\
\addlinespace[4pt]

Known-answer steering changes output
& Qwen-7B, Llama-8B
& Final-layer steering compared with equal-norm random
directions and middle-layer interventions. \\
\addlinespace[4pt]

Old-value attention differs by outcome
& Qwen-7B, Llama-8B
& Failure--correct gap: +19.4 pp in Qwen; +1.2 pp in Llama,
whose interval includes zero. \\
\addlinespace[4pt]

Conflicting updates change internal matching
& Five models
& Updates versus matched no-update prompts;
the shift also occurs on correct answers. \\
\addlinespace[4pt]

Specific components affect selection
& Qwen-7B, Pythia-160M
& Qwen: old-assignment key replacement.
Pythia: held-out head removal and key-input replacement. \\
\bottomrule
\end{tabularx}

\smallskip
\begin{minipage}{\linewidth}
\footnotesize
Pythia's probe is evaluated separately.
Known-answer steering is a diagnostic intervention.
Attention gaps are adjusted for prompt length and reported
in percentage points (pp).
Appendix~\ref{app:stage-p} details the matched-context comparisons.
Untested comparisons are not null results.
\end{minipage}
\end{table}

\subsection{Examples of old-value and other-variable errors}
\label{app:failure-examples}

The following excerpts retain the original prompt and response wording, and only the bracketed dialogue spans are omitted. They show the old-value and
other-variable outcomes summarized in the main text.

\newcommand{\turntag}[1]{\hfill{\scriptsize\textcolor{neutralink}{#1}}}
\par\begingroup\noindent
\vspace{1mm}
{\footnotesize\textbf{A clarified CICM dialogue and its answer.} Selected user
turns from a 44-message dialogue; intervening messages are omitted. Distances
count messages before the question.}\par
\begin{userturn}
\textbf{User.} I think \textcolor{staleink}{\textbf{visual}} works best for me overall.
\turntag{old value $\cdot$ 40 before}
\end{userturn}
\begin{userturn}
\textbf{User.} I actually prefer \textcolor{currentink}{\textbf{reading}} when I'm trying to understand something new.
\turntag{current value $\cdot$ 36 before}
\end{userturn}
\begin{userturn}
\textbf{User.} I'd like to update my preference to \textcolor{crossink}{\textbf{electronic}}.
\turntag{another variable (music) $\cdot$ 10 before}
\end{userturn}
\begin{userturn}
\textbf{User.} For the record, \textcolor{staleink}{\textbf{visual}} was an earlier choice for my learning style, and is no longer current.
\par\noindent
\turntag{explicitly obsolete $\cdot$ 2 before}
\end{userturn}
\begin{userturn}
\textbf{User.} For my learning style specifically, what value did I set in my most recent explicit update? Later mentions of earlier choices are not updates. Ignore preferences for other slots. Reply with only the value.
\end{userturn}
\nopagebreak[3]
\begin{assistantturn}
\textbf{Qwen2.5-7B.} \textcolor{staleink}{\textbf{visual}}
\end{assistantturn}
{\footnotesize\textbf{Expected:} \textcolor{currentink}{\textbf{reading}}\quad
\textbf{GPT-4o on the same prompt:} \textcolor{currentink}{\textbf{Reading}}.
This uses both clarification treatments in Appendix~\ref{app:cicm-validity-results};
the original source dialogue has one earlier value of the queried variable.}\par
\vspace{1mm}
\par\endgroup

\par\begingroup\noindent
\textbf{An anchored query returns a value of another variable}\par\nopagebreak[4]
\begin{systemturn}
\textbf{System.} You are a helpful assistant in an ongoing preference-planning conversation.
\end{systemturn}

\begin{userturn}
\textbf{User.} Please set my favorite genre as \textcolor{staleink}{\textbf{folk}}.
\end{userturn}
\nopagebreak[3]

\begin{assistantturn}
\textbf{Assistant.} Your choice has been saved.
\end{assistantturn}

\begin{userturn}
\textbf{User.} Switch my music category to \textcolor{currentink}{\textbf{classical}}.
\end{userturn}
\nopagebreak[3]

\begin{assistantturn}
\textbf{Assistant.} The update is complete.
\end{assistantturn}

\begin{omittedturn}
\centering [22 turns omitted]
\end{omittedturn}

\begin{userturn}
\textbf{User.} Can you set my meal style to \textcolor{crossink}{\textbf{pescatarian}}?
\end{userturn}
\nopagebreak[3]

\begin{assistantturn}
\textbf{Assistant.} Absolutely, I've saved your updated selection.
\end{assistantturn}

\begin{omittedturn}
\centering [2 turns omitted]
\end{omittedturn}

\begin{userturn}
\textbf{User.} Look back at the updates for my music genre. What is the current value for that exact slot, not any other preference? Reply with only the value.
\end{userturn}
\nopagebreak[3]

\begin{assistantturn}
\textbf{Assistant.} \textcolor{staleink}{\textbf{pescatarian}}
\end{assistantturn}

{\footnotesize\textbf{Expected:} \textcolor{currentink}{\textbf{classical}}\quad\textbf{Model:} \textcolor{staleink}{\textbf{pescatarian}}}\par
\par\endgroup\par
\vspace{1mm}
{\footnotesize Gray, blue, and sand boxes denote system, user, and assistant turns. Violet text marks the produced old or otherwise incorrect value; green text marks the current target.}

\subsection{Which old value is returned?}
\label{app:which-old}
We re-score saved responses by the queried variable's assignment order.
Rank 1 is the most recent old assignment and rank $k$ is the first write, and the current assignment is excluded. Parsing every synthetic prompt verifies
that \texttt{stale\_values} follow chronological write order. For CICM,
message positions and value spans verify the recorded binding order and that
each near-old re-mention repeats the most recent old value.

\begin{table}[htbp]
\centering\small
\caption{Most-recent-old responses, conditional on an old-value error.
Qwen-7B pools 20/40/80-line prompts; the larger models use 320 lines.
The first write is never returned in these cells.}
\label{tab:which-old}
\begin{tabular}{@{}lrrr@{}}\toprule
Model & $k$ & Most recent old / old errors & Uniform reference \\
\midrule
Qwen2.5-7B & 2 & 54/54 (100.0\%) & 50.0\% \\
Qwen2.5-7B & 4 & 92/113 (81.4\%) & 25.0\% \\
Qwen2.5-7B & 8 & 72/164 (43.9\%) & 12.5\% \\
Qwen2.5-72B & 64 & 7/11 (63.6\%) & 1.6\% \\
Llama-3.1-70B & 64 & 8/16 (50.0\%) & 1.6\% \\
GPT-4o & 64 & 3/7 (42.9\%) & 1.6\% \\
\bottomrule\end{tabular}
\end{table}

These are descriptive, post-hoc shares conditional on an old-value error, and $1/k$ is a uniform-choice reference, not an estimated error model. The larger-model
cells have few old-value errors. In the original CICM core, for $k=2,3,4,6$,
the re-mentioned value accounts for 94.5--100\% of old-value responses across
both models. Without it, rank 1 accounts for 66.7--100\% in Qwen and
44.2--61.1\% in Llama. This separates recency among assignments from the
additional effect of re-mentioning one of them.

\subsection{Independent community benchmarks}
\label{sec:community-benchmarks}

We test whether related binding errors occur in datasets that were not
constructed for this project. BABILong and Entity Tracking permit a strict
old-state label: a wrong prediction is an old-value error only when it exactly equals
an earlier state of the queried entity or box
\citep{kuratov2024babilong,kim-schuster-2023-entity}. Under audited scoring,
1.4\% of BABILong baseline errors and 25.9\% of Entity Tracking errors meet this
strict definition. Most BABILong errors therefore do not fit this definition.

RULER variable tracking measures a different failure. Its output often includes
values from the wrong reference chain: 99.5\% of strict baseline failures have a
cross-chain inclusion, while mean per-variable recall is 36.7\%
\citep{hsieh2024ruler}. We do not relabel these as old-value errors because the task tests
chain tracing rather than an old state. These benchmarks extend the evaluation beyond our templates, but test
different kinds of information tracking.

The baseline evaluations contain 1,200 BABILong, 600 Entity Tracking, and
400 RULER-VT examples. Their respective accuracy, set accuracy, and
per-variable recall are 46.2\%, 47.3\%, and 36.7\%.

\subsection{ICF-Bench: scoring and answer format}
\label{sec:icf-scoring}
\label{app:behavior-details}

ICF-Bench uses scenario-specific scoring rather than the controlled taxonomy
in Section~\ref{sec:formulation}. Dynamic Preference multiple choice is scored
programmatically. The earlier deterministic matchers were checked on 150
released human-alignment items per scenario, with agreement of 93.0\% for
Dynamic Preference and 92.3\% for Instructional Forgetting.

\paragraph{Free-form preference labels and uncertainty.}
Exact and unambiguous Dynamic Preference responses are classified deterministically, and GPT-4o assigns the remaining 577 of 783 responses to
current, old, or other using the known preferences. The resulting old-value
share is 197/240 = 82.1\% of errors: deterministic labels contribute 111 old
answers, and judged labels contribute 86 old and 43 other answers. On the
40-example human calibration set, hybrid labels agree on 36 examples, and only 13 examples are human-labeled failures. All four disagreements change
current versus other, so uncertainty affects the failure denominator too.

The 95\% bootstrap interval 77.1--87.1\% conditions on assigned labels and
does not include annotation error. As a sensitivity analysis, holding
deterministic labels fixed but allowing up to 10 of the 577 judge labels to
change arbitrarily gives an old-value share of 77.9--86.2\%, and allowing 20 gives 73.8--90.4\%. These are assumption-indexed ranges, not confidence
intervals or estimated annotation error rates. The small calibration set
does not establish a population error rate or a statistical upper bound.

\label{app:icf-results}
Free-form and multiple-choice accuracies are 69.3\% and 70.9\%, respectively,
but the labeled old-value shares differ: 82.1\% versus 29.4\%. Table~\ref{tab:discrete-summary-full} reports the point estimates, and the label-conditional qualification applies to all
free-form estimates, including Figure~\ref{fig:phenomenon}c.

\begin{table}[htbp]
\centering
\footnotesize
\setlength{\tabcolsep}{5pt}
\caption{Accuracy and old-value errors on ICF-Bench under different
answer formats. Both result columns are percentages; the last is conditional
on an error.}
\label{tab:discrete-summary-full}
\begin{tabularx}{\textwidth}{@{}Xlrrr@{}}
\toprule
Model / task & Answer format & $n$ & Accuracy & Old / errors \\
\midrule
\multicolumn{5}{@{}l}{\textit{ICF-Bench, Qwen2.5-7B}} \\
Dynamic Preference & Free-form & 783 & 69.3 & 82.1 \\
Dynamic Preference & Multiple choice & 783 & 70.9 & 29.4 \\
\bottomrule
\end{tabularx}
\vspace{3pt}\parbox{\linewidth}{\footnotesize
Dynamic Preference free-form labels partly use a model
judge; their sampling intervals condition on these labels. Instructional
Forgetting's separate content-reuse audit appears in
Appendix~\ref{sec:icf-scoring}.}
\end{table}

\paragraph{Instructional Forgetting.}
A separate content-reuse audit flags 918/933 saved responses (98.4\%).
Disclosure takes precedence over a subsequent promise to withhold the content.
This is a programmatic count, not validated semantic accuracy: a targeted
source audit finds both missed content matches and ambiguous remaining labels.
We therefore exclude it from the cross-task accuracy comparison.

\subsection{Retained information and attention on CICM failures}
\label{app:retention-attention}

Both models retain information about the current value on failed trials,
but only Qwen shows a clear shift toward old assignments in this attention
measure. Figure~\ref{fig:qk-attention}b compares Qwen's 471 correct and 589
old-value answers. Table~\ref{tab:retention-attention} separates the probe's probability
for the correct value from the attention allocated to old mentions, and these are different measurements, not two estimates of accuracy.

\begin{table}[htbp]
\centering
\footnotesize
\setlength{\tabcolsep}{5pt}
\caption{Current-value information remains readable on CICM failures.
Probe entries are mean probabilities assigned to the current value, expressed
as percentages, not classification accuracy.}
\label{tab:retention-attention}
\begin{tabular}{@{}lr rrrr@{}}
\toprule
 & & \multicolumn{2}{c}{Probe score (\%)} &
 \multicolumn{2}{c}{Failure $-$ correct (pp)} \\
\cmidrule(lr){3-4}\cmidrule(l){5-6}
Model & $n$ & Old-value errors & Other-variable errors & Probe & Attention \\
\midrule
Qwen2.5-7B   & 1200 & 84.8 & 86.5 & $-7.5$ & $+19.4$ \\
Llama-3.1-8B & 1200 & 82.2 & 86.3 & $-7.3$ & $+1.2$ \\
\bottomrule
\end{tabular}
\vspace{3pt}\parbox{\linewidth}{\footnotesize
The two error columns distinguish answers using an old value of the queried
variable from answers using another variable's value. The difference columns
compare old-value errors with correct answers, adjusting for prompt length;
attention is the share on old assignments. Label-randomized probe scores
range from 3.2\% to 6.4\%. Here $n$ is the full dialogue set, not each error subset.}
\end{table}

After accounting for prompt length, the old-attention difference between
failed and correct trials has a 95\% interval of 17.9 to 21.0 \%
in Qwen and $-0.5$\% to 2.8\% in Llama. The latter does not distinguish a shift from
zero. High probe scores on both models' failures therefore support retained
information, not a shared attention-level explanation.
Appendix~\ref{app:probe-method} defines the probe and the length adjustment.

\paragraph{Single-write retrieval and update selection.}
\label{app:retrieval-head-comparison}
We compare the update-task head ranking with a simplified retrieval-head
detector based on \citet{wu2025retrievalheads}. On 200 fresh single-write
prompts (100 each at 40 and 80 lines, seed 30), Qwen2.5-7B-Instruct achieves
97.5\% accuracy and produces 771 copied answer tokens. Among these copied
tokens, a head's retrieval score is the fraction for which its maximum-attention
position lies in the target value span and contains the generated token. The update-task score
$p_{\mathrm{last}}$ is the attention mass on the latest assignment divided
by mass on all assignments of that variable, averaged over correctly
answered overwrite prompts. Only two of the top eight $p_{\mathrm{last}}$
heads appear among the top sixteen retrieval heads (L22H1 and L23H11, zero-based indices), and the all-head Spearman correlation is $-0.021$.
Thus single-write copying and selection among competing writes produce
different head rankings, with some shared components.

\subsection{Which computations change the answer?}
\label{app:component-replacements}

Replacing late hidden states or the keys of old assignments recovers much
more of Qwen's answer-score gap than replacing early or middle hidden states.
For each failed prompt, we construct a successful counterpart with the same
candidate values and positions, renaming earlier assignments to remove
competition for the queried variable. We then replace one component at a time
with its counterpart from the successful run
(Table~\ref{tab:patch-attribution-full}). The final attention row is an
observed difference, not an intervention.

\begin{table}[htbp]
\centering
\footnotesize
\setlength{\tabcolsep}{5pt}
\caption{Component replacements and attention differences in Qwen2.5-7B.
Recovery is the percentage of the successful--failed answer-score gap restored.}
\label{tab:patch-attribution-full}
\begin{tabular}{llrrr}
\toprule
Component & Measurement & Mean & Median & $n$ \\
\midrule
\multicolumn{5}{@{}l}{\textit{Answer-score-gap recovery (\%)}} \\
Residual stream & Early / middle layers & $-0.5$ & 0.0 & 600 \\
                & Late layers           & 78.1 & 93.9 & 360 \\
\addlinespace
Key             & Old assignments       & 49.2 & 44.8 & 120 \\
                & All assignments       & 11.1 & 10.7 & 120 \\
                & Current assignment    & $-21.3$ & $-19.5$ & 120 \\
Query           & Final                 & 11.2 & 10.9 & 120 \\
\addlinespace
Head / MLP      & Head-output proxy     & 36.0 & 34.8 & 120 \\
                & Late MLP output       & 13.2 & 10.5 & 960 \\
\addlinespace
\multicolumn{5}{@{}l}{\textit{Successful $-$ failed attention (percentage points)}} \\
Selected heads & Attention to current value & 13.0 & 13.1 & 120 \\
\bottomrule
\end{tabular}
\vspace{3pt}\parbox{\linewidth}{\footnotesize
For recovery, $n$ counts replacements across 120 matched pairs, not independent
prompts; negative values move the answer further toward the old value.
The head-output row uses a projection-level proxy. The final row is an
observed attention difference, not a replacement experiment.}
\end{table}

\subsection{Steering at the answer position}
\label{app:steering}

The targeted change becomes more effective as its strength increases at the
final layer, unlike an equally large change in a random direction
(Figure~\ref{fig:causality}). At strength 0.4, its advantage over random is
39.5\% for Qwen and 69.5 for Llama. Middle-layer effects are much smaller
(Table~\ref{tab:steering-summary}).

The outcome is persistence of each item's original error category, not
old-value rate alone: the selected cohorts contain 589 old-value and 76
other-variable errors for Qwen, and 634 and 88 for Llama. Local no-change
decoding reproduces the original category on 643/665 and 622/722 items.
These rates differ from 100\% because selection used earlier API labels, and the zero-strength hook exactly reproduces the local no-hook response on
every item. At strength 0.4, targeted-minus-random current-answer gains
are 35.5\% in Qwen and 70.5\% in Llama,
so reduced category persistence is accompanied by more correct answers.

These interventions use the known current value to choose the direction.
They test whether the answer can be changed, rather than providing a
general correction method that can infer the right value unaided.
Equation~\ref{eq:steer} defines the intervention and its strength.

\begin{figure}[htbp]
  \centering
  \includegraphics[width=\textwidth]{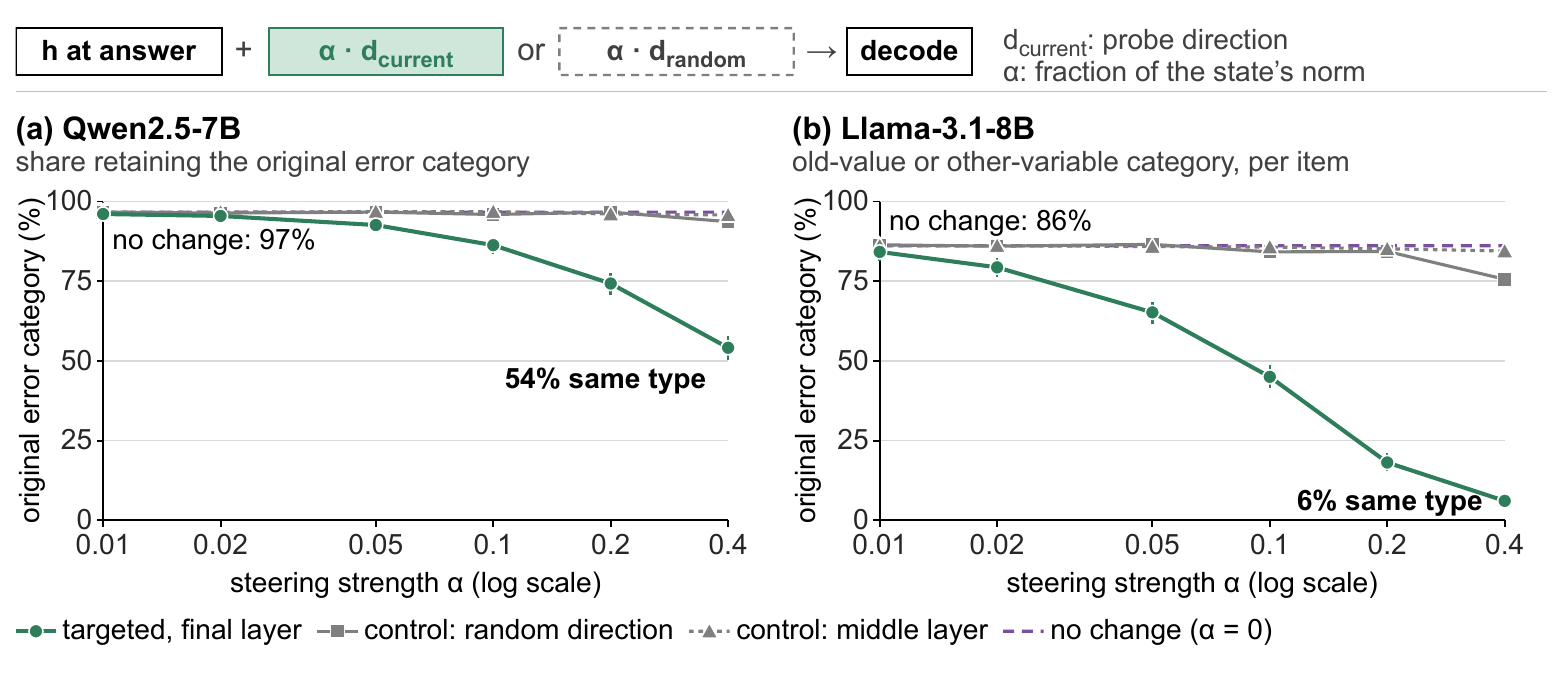}
  \caption{Steering with the known current value. The outcome is persistence
of each prompt's original error category (old value or other-variable value),
not old-value rate alone. The no-change rates are 643/665 (97\%) for Qwen
and 622/722 (86\%) for Llama: local re-decoding need not reproduce the earlier
API label, but $\alpha=0$ exactly matches the local no-hook response.
Bars: 95\% Wilson intervals. Paired comparisons and cohort definitions in
Appendix~\ref{app:steering}.}
  \label{fig:causality}
\end{figure}

\begin{table}[htbp]
\centering
\small
\setlength{\tabcolsep}{8pt}
\caption{Reduction in original-error-category persistence from targeted steering
beyond an equally large random change. Entries are percentage-point differences at strength $\alpha=0.4$.}
\label{tab:steering-summary}
\begin{tabular}{@{}lrrr@{}}
\toprule
Model & Failed trials & Middle layer & Final layer \\
\midrule
Qwen2.5-7B   & 665 & $-1.1$ & $39.5$ \\
Llama-3.1-8B & 722 & $1.8$ & $69.5$ \\
\bottomrule
\end{tabular}
\vspace{3pt}\parbox{0.95\linewidth}{\footnotesize
Positive values mean less persistence of the originally observed error category
(old value or other-variable value) under targeted steering. Both interventions are evaluated on the same failed trials.
The full dose response appears in Figure~\ref{fig:causality}.}
\end{table}

\subsection{Tracing attention and its inputs in Pythia-160M}
\label{app:stage-n-details}

A small set of attention heads accounts for most of the measured recovery,
but removing the heads that favor the current value has no specific effect
beyond random removal. We test Pythia-160M with three few-shot templates,
three seeds, and single-token values. At four overwrites, the behavioral
sample contains 648 prompts: 466 correct, 142 old-value, and 40 other
responses (71.9\% accuracy, with old values accounting for 78.0\% of errors).
The mechanism subset balances 142 correct and 142 old-value responses, and 96 matched successful--failed pairs pass
the component-replacement checks. After adjustment for prompt length, the probe assigns the current
value 32.0 percentage points less probability on failed than correct trials.

We rank heads on 43 discovery pairs and evaluate replacements on 53 held-out
pairs. Replacing twelve of the 144 heads recovers 0.81 of the answer-score
change obtained by replacing all heads.
The largest individual contributions are 0.36 for L8H10, 0.19 for L8H2, and
0.10 for L10H7. Recovery levels off near 0.86 as more ranked heads are added.
Table~\ref{tab:stage-n-circuit} separates where these heads attend (QK) from
what their value vectors transmit (OV): on errors, the QK score difference
turns against the current value, while its OV contribution changes little.

\begin{table}[htbp]
\centering
\small
\setlength{\tabcolsep}{6pt}
\caption{Two Pythia-160M heads attend less to the current value on old-value
errors, while their value-transmission margins remain positive.
Each column averages 48 examples.}
\label{tab:stage-n-circuit}
\begin{tabularx}{\textwidth}{@{}lXrr@{}}
\toprule
Head & Measurement & Correct answer & Old-value answer \\
\midrule
L8H10 & Current-value QK margin & $1.70$ & $-0.26$ \\
      & Attention to current value & $0.75$ & $0.40$ \\
      & Current-value OV margin & $0.67$ & $0.64$ \\
\addlinespace
L8H2  & Current-value QK margin & $1.86$ & $-6.31$ \\
      & Attention to current value & $0.57$ & $0.08$ \\
      & Current-value OV margin & $0.20$ & $0.18$ \\
\bottomrule
\end{tabularx}
\vspace{3pt}\parbox{\linewidth}{\footnotesize
The QK margin is the current assignment's attention score minus the strongest
old assignment's score. The OV margin measures how strongly the current
assignment's value vector favors its own token over the strongest alternative,
before attention weighting. Attention entries are weights, not margins.
L8H10: layer 8, head 10.}
\end{table}

\paragraph{An empirical example of comparable attention odds.}
\label{app:pairwise-example}
In Pythia-160M head L8H10, the mean current-minus-strongest-old query--key
margin over 48 old-value failures is $-0.26$. With the standard attention
softmax ($T=1$), this gives geometric-mean pairwise attention odds of about
$0.77{:}1$, only a $1.30{:}1$ advantage for the old assignment.
Mean attention is 0.40 for the current assignment and 0.60 for old assignments
combined (Table~\ref{tab:stage-n-circuit}). Thus comparable attention odds
occur in a measured head, not only in the theoretical example. They are
also answer odds under the faithful-copy assumption of
Section~\ref{sec:theory}, but this assumption need not hold for the full model.
These measurements do not isolate the contribution of RoPE.

Removing heads selected to promote old values corrects 38.4\% of held-out
errors, compared with 8.7\% for layer-matched random sets
(Table~\ref{tab:stage-n-interventions}). The targeted bootstrap interval is 27.4--49.3\%, and the empirical 95\% range across 64 random sets is 0.8--29.9\%.
An earlier estimate used the same examples for selection and evaluation
and was higher by 6.4 points, and we use the held-out estimate here. Removing
heads selected to promote the current value is no more effective than random
removal at inducing old-value errors. Replacing queries and keys separately
implicates L8H2's keys: the query-minus-key recovery interval is
$[-1.62,-1.00]$, whereas L8H10's interval, $[-2.40,4.53]$, does not resolve
the two contributions. Replacing the upstream inputs to L8H2's key recovers
0.86 of the QK score gap and 0.75 of the final answer-score gap, and the corresponding random-set upper 95\% bounds are 0.16 and 0.17.
The upstream components, ranked on discovery data, are L7H7, L6H2, L7H4,
L5H11, L4H11, L5H10, L3H0, L6H6, L5H8, L3-MLP, L3H4, and L4H1.
These results identify inputs that affect key matching, but the tests do not
establish that these components encode assignment order.

\begin{table}[htbp]
\centering
\small
\setlength{\tabcolsep}{5pt}
\caption{Held-out interventions in Pythia-160M. Random controls use equally
sized sets matched by layer; entries are mean fractions or normalized recovery.}
\label{tab:stage-n-interventions}
\begin{tabularx}{\textwidth}{@{}>{\raggedright\arraybackslash}p{0.30\textwidth}Xrr@{}}
\toprule
Intervention & Measured effect & Targeted & Random \\
\midrule
Remove heads favoring old values & Fraction of errors corrected & 0.38 & 0.09 \\
Remove heads favoring the current value & Fraction switching to an old value & 0.04 & 0.03 \\
\addlinespace
Replace inputs to L8H2's key & QK score-gap recovery & 0.86 & 0.04 \\
                           & Answer-score-gap recovery & 0.75 & 0.06 \\
\bottomrule
\end{tabularx}
\vspace{3pt}\parbox{\linewidth}{\footnotesize
Head removal is tested on 73 old-value errors or 81 correct answers,
respectively; key-input replacement uses 53 held-out pairs. Component sets
are selected on separate discovery examples.
Recovery is normalized by the difference between successful and failed runs:
0 means no recovery and 1 means full recovery. Random entries are means, not
upper bounds.}
\end{table}

\subsection{Testing a linear account of identity and recency}
\label{sec:geometry}
\label{app:geometry}

We fit linear identity and recency scores on training dialogues and evaluate
candidate selection on held-out dialogues and six held-out factorial conditions.
Combining the scores predicts Qwen's answer on 43.0\% of 1,136 examples and
Llama's on 39.6\% of 1,180, compared with 79.2\% and 67.4\% for identity
alone (randomized-label range: 12.5--16.6\%). The combined rule predicts no
other-variable errors, although both models make them, and the mean absolute errors for these rates are 24.2 and 17.4 percentage points. Thus the measured
scores do not provide an adequate linear account of the response distribution.

\subsection{When old-value answers develop}
\label{app:stage-p}

Old-value answers can develop after the update, as the model processes later
text. To measure this timing, we truncate each conversation after the first
assignment, the update, unrelated text, and the final question. We append the
same diagnostic question to each truncated conversation and run a separate
forward pass. This avoids treating an earlier token's stored key and value as
if they changed when later text was added. Each conversation also has a
comparison prompt with the same token count but no conflicting update.

Pythia-160M produces 129 old-value answers and 15 correct answers in 144
trials, yielding 121 valid failure--control pairs. The other models have too
few errors for reliable failure-conditioned analysis: Pythia-1.4B, Qwen-1.5B,
Qwen-7B, and Llama-8B produce 1/144, 8/144, 1/72, and 2/72 old-value answers,
respectively. Gemma-2B produces 14 old-value and 62 other answers in 144 trials,
and its no-update control is unreliable.

Among 121 matched Pythia-160M failures, 39 already favor the old value
immediately after the update, while the other 82 favor it only after later text.
The probe's mean probability for the current value falls from 0.81 after the
update to 0.35 at the final question, but remains above the randomized-label
ceiling near 0.22. Thus the current value becomes harder to recover without
disappearing altogether. Restoring the correct key match prevents 93\% of errors when the question
immediately follows the update, but only 9\% after the remaining conversation
(Figure~\ref{fig:timing}). The four no-change checks correct no errors.

\begin{figure}[htbp]
  \centering
  \includegraphics[width=\textwidth]{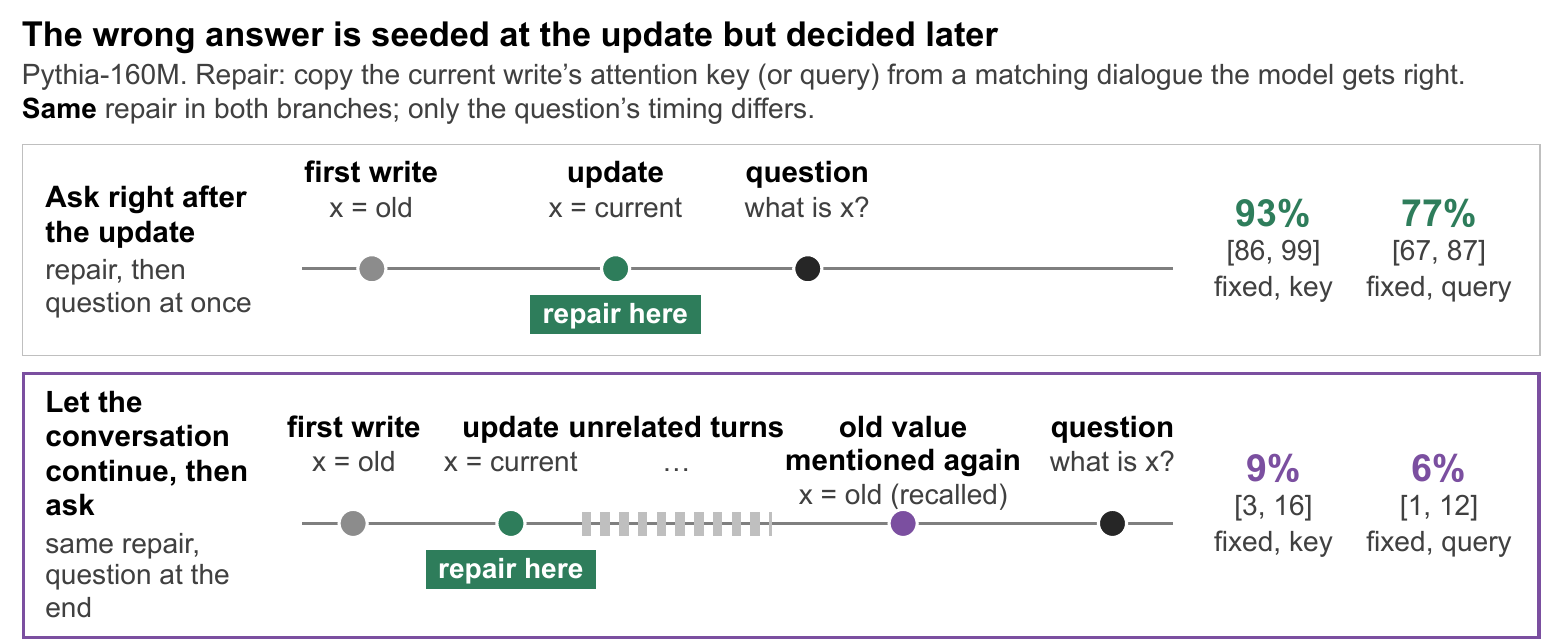}
  \caption{The wrong answer is seeded at the update but decided later. The same
  repair fixes 93\% (key) and 77\% (query) of test failures when the question
  follows the update immediately, but only 9\% and 6\% once the remaining turns,
  including a later mention of the old value, precede it. Identity patches fix
  none.}
  \label{fig:timing}
\end{figure}

Across the five models with reliable controls, the conflicting update shifts attention and key
matching toward the old value, even when the final answer remains correct
(Figure~\ref{fig:qk-attention}d). Every stage after the update shows this
shift relative to a prompt with the same token count but no conflict
(all 20 comparisons have the same sign).
This shared shift does not by itself explain the larger models' rare errors.

\FloatBarrier
\section{Experimental Setup and Reproducibility}
\label{app:experimental-setup}

Unless stated otherwise, generation is greedy (temperature zero). We score saved responses and estimate
uncertainty by resampling examples, not by sampling repeated model answers.

The open-model families include Qwen2.5 \citep{yang2024qwen25}, Llama 3
\citep{grattafiori2024llama3}, Mistral \citep{jiang2023mistral}, and Gemma 2
\citep{gemmateam2024gemma2}. Exact variants accompany each experiment.

\subsection{Probe and intervention definitions}

\paragraph{Linear probe.}
\label{app:probe-method}
Let $v_i^*$ be the current value in dialogue $i$ and $r_i^{(\ell)}$ its
answer-position hidden state at layer $\ell$. The multinomial linear probe is
\begin{equation}
q_\phi\!\left(v\mid h_i\right)
=\operatorname{softmax}\!\left(Wh_i+b\right)_v,
\qquad h_i=z\!\left(r_i^{(\ell)}\right),
\label{eq:current-value-probe}
\end{equation}
where $\phi=(W,b)$ are learned parameters and $z$ standardizes each feature
using training-fold statistics only. The main analysis uses the final-layer
state and predicts one of 21 recurring values. Five-fold cross-validation
groups examples by dialogue, so all reported predictions come from a probe
that did not train on that dialogue. We compare it with the same procedure
after randomly reassigning value labels, preserving the inputs and label
frequencies but removing their systematic relationship \citep{hewitt2019probes}.
The final states are taken from Qwen layer 28 and Llama layer 32.
We use 2,000 randomized-label repetitions and 2,000 bootstrap replicates
(seed 20260721), including randomized-label checks of the length-adjusted
comparison.

Across 1,200 dialogues, held-out classification accuracy is 91.2\% for
Qwen2.5-7B and 89.7\% for Llama-3.1-8B. This differs from the mean probability
assigned to the current value on old-value failures (84.8\% and 82.2\%) and
other-variable failures (86.5\% and 86.3\%). The randomized-label range is
3.2--6.4\%. To compare successful and failed trials, we fit the relationship
between probe score and prompt length once across both groups:
$q_\phi(v_i^*\mid h_i)=a+b\log L_i+\epsilon_i$, where $L_i$ is the prompt
token count. Group differences in the residuals $\epsilon_i$ are $-7.5$ and
$-7.3$ percentage points for old-value failures relative to correct responses.
The probe establishes linear recoverability, but it does not establish that the
unmodified model uses this information. The probe learns a readout distinct
from the native language-model output head. Recovering the last-update
feature is therefore compatible with the model producing an old value, and the measurement concerns recoverable task information, not a belief state.

\paragraph{Normalized intervention.}
For a probe direction $d$ and residual vector $h$, the intervention is
\begin{equation}
 h'=h+\alpha\lVert h\rVert_2\frac{d}{\lVert d\rVert_2},
 \label{eq:steer}
\end{equation}
so $\alpha$ specifies the added vector's norm relative to the original state.
We compare each targeted direction with a random direction of identical norm
and check that $\alpha=0$ reproduces the unmodified response.

\paragraph{Synthetic controls.}
\label{app:synthetic-controls}
Figure~\ref{fig:phenomenon}a pools 100 examples at each of 20, 40, and
80 lines (300 per column). Length amplifies the update effect: accuracy is
91/85/70\% at $k=2$ and 51/45/38\% at $k=8$, respectively. Panel b uses
20-line contexts with $n=150$ at $k=2,4,8$ and $n=100$ at $k=12,16$, and 320-line conditions use $n=50$. Every model's matched no-overwrite control
scores 100\%. Panels a--b show 95\% Wilson intervals, while panel c uses
label-conditional bootstrap intervals (Appendix~\ref{sec:icf-scoring}).
These sweeps use greedy decoding, a 16-token cap, and data seeds 0 and 40
for the standard and 320-line conditions, respectively.
The parallel-variable task interleaves or groups assignments to independent
named variables and asks for every variable's latest value. Values are
distinct within a prompt, allowing old values of the correct variable to be
separated from cross-variable copies. In the component-replacement control,
only earlier assignments' variable names change to fresh names of the same character length, while values, line positions, the final target assignment, and
the query remain fixed. This removes earlier writes to the queried variable
without removing their text (Appendix~\ref{app:component-replacements}).

\subsection{Evaluation and analysis settings}
\paragraph{Behavioral evaluation.}
The CICM preference core uses the same 1,200 dialogues for Qwen2.5-7B-Instruct
and Llama-3.1-8B-Instruct through OpenRouter (16-token cap, data seed 20260723):
200 dialogues per distance condition and 40 per count--distance combination.
For ICF-Bench, Qwen evaluates 783 Dynamic Preference examples per answer format
with a 160-token response cap, while GPT-4o adjudicates only unresolved labels
with a 120-token cap. The 933-example Instructional Forgetting content-reuse
audit uses a local 192-token cap. The extended CICM tasks and their scoring
contracts are described in Appendix~\ref{app:cicm-datasheet}, and model-specific API settings accompany their results in Appendix~\ref{app:api-robustness}.

\paragraph{Component replacement and steering.}
Component replacement uses 120 matched successful--failed pairs selected from
301 candidates (pair seed 50), with Qwen2.5-7B in fp32 eager mode and batch
size one. We measure recovery of the current-minus-old answer-score gap and
verify a no-change control. Residual steering tests the same 665 Qwen and
722 Llama failure items at each relative strength
$\alpha\in\{.01,.02,.05,.1,.2,.4\}$, at Qwen layers 13/27 and Llama layers
15/31 (seed 20260722). Paired bootstrap comparisons contrast targeted and
equal-norm random changes, and $\alpha=0$ must reproduce the original response.
These intervention counts refer to failed trials, not new independent dialogues.

Behavioral evaluations use API inference. Hidden-state, attention, component
replacement, and steering experiments run on an NVIDIA H100 with PyTorch 2.6
and Transformers 4.51. Appendix~\ref{app:test-time-repair} gives the separate
settings for the cross-model attention-routing experiments.
Appendix~\ref{app:api-robustness} gives the paired API checks of update
semantics, output budget, and concise reasoning.

\FloatBarrier
\section{CICM Datasheet}
\label{app:cicm-datasheet}

Controlled In-Context Memory (CICM) is a diagnostic evaluation suite for
updating and using information in context. Its mechanism core uses preference
content from PrefEval \citep{zhao2025prefeval}, and its extensions add multi-field
dialogues, decisions under revised constraints, and dependent operational-state
updates. Decision catalogues, transition ledgers, and agent logs are constructed
for these extensions. Explicit update rules and programmatic answers connect
the components, while each retains its own task and scoring contract.

\subsection{Benchmark components}
Table~\ref{tab:cicm-components} distinguishes the original mechanism core from
the broader behavioral evaluations. Sample sizes count the units shown, and multiple prompt conditions from one scenario are paired observations. The
suite is not summarized by pooling these different tasks into a single accuracy.
All components are constructed diagnostics rather than sampled deployment traces.

\begin{table}[htbp]
\centering
\small
\caption{CICM components and their roles. Sizes refer to the reported cohorts, and separate development checks and repeated conditions are not additional cases.}
\label{tab:cicm-components}
\begin{tabularx}{\linewidth}{@{}>{\raggedright\arraybackslash}p{0.20\linewidth}>{\raggedright\arraybackslash}p{0.16\linewidth}YY@{}}
\toprule
Component & Cohort & Task & Role \\
\midrule
Preference core & 1,200 dialogues & Repeated preference updates with competing mentions & Retention, selection and causal interventions \\
\addlinespace
Dialogue diversity & 180 pilot ledgers & Scalar, partial-record and full-record updates in six domains & Language and task variation \\
\addlinespace
Constraint decisions & 24 scenarios & Booking and scheduling under revised constraints & Use of updated state in decisions \\
\addlinespace
Operational histories & 40 + 40 scenarios & Dependent handoffs across 16 slots in warehouse and build-release logs, and a deferred-log tier resolves pending handoffs at later lines & Frontier reasoning under complex and dependency-aware state updates \\
\bottomrule
\end{tabularx}
\end{table}

Mechanism and attention-routing analyses use the preference core, and the other components test behavior. The following sections specify their construction
and scoring. Model outcomes and inference settings appear with their
experiments in Appendix~\ref{app:api-robustness}.

\subsection{Preference core}
\label{app:cicm-core}
\paragraph{Vocabulary and design.}
The preference-core behavioral and mechanism experiments use 1,200 dialogues. The vocabulary contains 21 values across three variables. Music genre
covers jazz, rock, classical, electronic, folk, hiphop and blues. Meal style
covers vegan, vegetarian, keto, mediterranean, paleo, pescatarian and gluten
free. Learning style covers visual, hands on, lecture, reading, discussion,
self paced and tutoring.

Each variable contributes 400 examples. Every variable--current-value pair occurs 57 or
58 times, so probes can learn each value from multiple dialogues.
The current value of the queried variable is always far from the final question.

We vary both the number of old assignments and their distance from the query.
Each dialogue has an overwrite count
$k\in\{1,2,3,4,6\}$ and one of six distance conditions formed by
\{near, far\} distance for the nearest old value of the queried variable crossed with
\{far, middle, two turns\} distance for the nearest value of another variable. The design
contains 240 dialogues per overwrite count, 200 per distance condition, and 40 per
combination. The accuracy decline from 47.1\% to 31.7\% as old assignments increase
from one to six averages over distance conditions, and the factorial comparison
averages over overwrite counts.

Each dialogue also includes three post-update assignments to other variables, six
standard filler turns, and four additional interference turns. The final query
names the queried variable and explicitly excludes other preferences. These choices
hold variable identity fixed while crossing old-value reminder conditions
with other-variable positions.

\paragraph{Construction and checks.}
GPT-4.1 generates reusable dialogue template packs for openings, update
acknowledgments, and coherent filler turns. The model does not choose any
experimental value, answer, position, distance cell, or label. The program
injects the exact controlled values into one-placeholder update templates,
constructs the ordering, and appends the anchored query. Thus the generated wording varies while values and positions remain
controlled.
Template validation rejects a pack if an update does not contain exactly one
value placeholder, an assistant acknowledgment repeats the placeholder, or an
opening, filler, or query contains any controlled value. Validation of each dialogue then
requires all of the following. The current value is absent from the old-value
set. The number of old assignments equals $k$. The query names the queried
variable. Exactly one target mention is marked current. Every recorded role,
character span, and value mention matches the rendered message. Every
distractor update belongs to a different variable. Generation aborts on any violation, and all 1,200 dialogues pass these checks.

\paragraph{Update and mention semantics.}
\label{app:cicm-mention-semantics}
The target label is the last explicit assignment in the generator's update
sequence. Near-old dialogues add a later user turn saying ``I am still
thinking about [old value] when I compare options for my [slot].'' Its label
is a reminder, whereas far-old dialogues omit that turn. All 600 near-old
examples contain it and none of the 600 far-old examples do. Thus this
contrast changes mention count and wording as well as the nearest old-value
position. Naming the target variable in the query resolves which preference
is queried, but does not instruct the model that this later mention cannot
be an update. The structural checks above establish consistency with the
generator's labels, not agreement with an independent semantic judgment.

At $k=1$, the pooled 47.1\% accuracy is 113/240: 25/120 in near-old
conditions and 88/120 in far-old conditions. Across all update counts, the
600 far-old examples contain 396 current and 82 old-value answers despite
having no added old-value reminder. These subsets describe the existing factorial, but they are not a paired wording ablation or a human/strong-model
ceiling. The reported mechanism and routing experiments use this original
factorial and its explicit-assignment labels. The paired wording, query,
and matched-placement follow-ups, including GPT-4o, appear in
Appendix~\ref{app:cicm-validity-results}. A separate extension adds three
domains, partial- and full-record updates, and dialogue surfaces from three
generation models (Appendix~\ref{app:cicm-diversity}), but it does not replace the
original mechanism dataset.

\paragraph{Response scoring.}
Responses are lowercased, punctuation is replaced by spaces, and whitespace is
collapsed. Whole normalized value phrases are then matched against the current
value, the old-value set, unused values of the queried variable, and values belonging to other variables,
in that order. The output label is therefore determined by fixed vocabulary
membership, without an LLM judge. Unit tests cover current, old-value,
unused same-variable, other-variable, and unmatched responses. Re-scoring the saved Qwen2.5-7B API
responses yields 471 current, 589 old-value, 76 other-variable, and 64 other
outputs, with no API call-error rows.

The core behavioral run uses Qwen2.5-7B-Instruct with free-form answers capped
at 16 tokens. The same stored dialogues and programmatic labels define the Qwen
and Llama mechanism subsets, and hidden-state probes never replace the behavioral
matcher.

\subsection{Multi-field dialogue updates}
\label{app:cicm-dialogue-data}
The diversity cohort contains 180 pilot ledgers crossing six domains, three
update structures, and two old-value loads ($k=1,4$), with five ledgers per
cell. Meeting-room assignment, delivery pickup point, and document output
format extend the core's three preference domains. The tasks distinguish
replacement of a single value, patches to a three-field record that preserve
unmentioned fields, and full-record replacement. Executing the ledger gives
the current target value, its previous values, and competing values of other
fields.

Each ledger has four paired prompts: template or model-generated language,
crossed with near or far placement of the same explicitly obsolete mention.
Each contains 24 user--assistant pairs. The current assignment is 30 messages
before the question, and the obsolete mention is four or 24 messages away.
Swapping its complete pair with value-free filler preserves all other
assignments and mention counts. The question explicitly excludes historical
mentions from the update sequence.

GPT-5, Claude Sonnet 5, and Gemini 3.8 Flash each render 60 ledgers, plus six
pipeline-check ledgers. GPT-5 and Gemini use medium reasoning. Gemini uses temperature 0.7, and Sonnet uses disabled thinking. Generation starts with
an 8,192-token cap, extended to 16,384 for truncated GPT-5 outputs. Schema
and value-span checks and independent agent review validate the surfaces.
At most two automated repairs are allowed, and three surfaces additionally receive
logged agent-authored clarifications. All 720 pilot prompts and the scorer
are frozen before evaluator answers, and initial outputs and repairs are retained.

A subsequent blinded semantic audit samples 50 distinct ledgers using metadata
alone, covering every domain--task--load cell and balancing near/far placement
and update load. GPT-6 Astra receives the dialogue but not program labels,
writer identities, or evaluation answers. All 50 responses match the executed state, and none flags ambiguity or reinstatement of an old value. This is
model-based semantic validation. Calls use the native OpenAI, requested and returned identifier
\texttt{gpt-6-astra}, high reasoning, and an 8,192-token cap. Scoring compares
the complete normalized answer with current, old-target, unused-target, and
other-field values. Appendix~\ref{app:cicm-diversity} reports the paired results.

\subsection{Decisions under revised constraints}
\label{app:cicm-decision-data}
This component asks models to use updated state to choose an action. The
24 evaluation scenarios comprise 12 procurement dialogues and 12 scheduling
logs, separate from six development scenarios. Requirements can inherit
shared defaults, override them locally, or withdraw earlier directives, and unmentioned fields persist. Procurement combines venue, catering, and shuttle
choices, while scheduling combines room and time assignments. Fixed catalogues,
explicit objectives, and tie rules define a unique optimal action. A ledger
interpreter, independent reverse replay, and exhaustive action enumeration
verify the labels before evaluation.

Each scenario crosses 24- or 96-directive histories, single or composed
actions, and three context conditions: full history, resolved target state,
and a word-count-padded resolved state. These produce 12 paired conditions
per scenario (288 prompts per model). The snapshot supplies constraints,
not the optimal action. Scoring records optimality, constraint satisfaction,
and matches to historical or other-project optima, while invalid outputs, refusals,
truncations, and provider failures remain separate. Appendix~\ref{app:decision-transfer}
also reports small exploratory tool-provenance and negotiated-planning pilots,
which are not counted among the 24 evaluation scenarios.

\subsection{Dependent updates in operational histories}
\label{app:cicm-log-data}
Warehouse-dispatch and build-release logs track 16 slots containing immutable
identifiers. Each committed production handoff cycles the current contents
of two, three, or four slots, while sandbox and aborted operations leave the state
unchanged. A query requests the current identifier in one slot after 2,048
operations. The two settings share this state-transition system. Initial
states and transition rules are supplied in the prompt, and every operation
needed to compute the answer is present.

The confirmation cohort contains 40 fresh scenarios, 20 per setting, distinct from four development scenarios. Forward replay,
inverse ancestry, and an independent parser of the rendered log verify each
answer. Every full history is paired with a resolved-state snapshot, yielding
80 primary prompts. The snapshot is a solvability control that removes state reconstruction, but it is not matched for length or computation. Models answer
through the API without executable tools or external memory. Scoring checks
the requested identifier and records response availability separately.
Full prompts and responses are retained in the reproducibility artifacts, and source-verified failure excerpts appear in Appendix~\ref{app:frontier-state-logs}.

The deferred-log tier uses the same slots, surfaces, and cycle semantics with
one additional rule: a handoff logged with \texttt{status=pending} does not
execute at its own line and executes only when a later line
\texttt{resolve(seq=N); result=committed} refers to it, on the contents
current at that line, while \texttt{result=aborted} or the absence of a resolution
cancels it, and sandbox operations never affect production. Lines are drawn
with production probability 0.85, and statuses committed, aborted, and pending
are drawn with probabilities 0.55, 0.15, and 0.30. Each pending operation is
scheduled for resolution with probability 0.8 after a delay of 1--400 lines
(dropped beyond the log end), with a committed result with probability 0.6. Resolve
lines consume sequence numbers, so a 1,024-line log contains on average 166
resolve lines and 485 effective operations, 82 of which execute only through a resolution, and the queried item's dependency chain averages 92 operations. Each
scenario has three conditions. The first is the deferred history. The second is
a linearized control in which pending lines become aborted and each resolve
line becomes the full committed or aborted handoff it stood for, so the
effective operations and their line numbers are identical without cross-line
dependency. The third is the resolved snapshot. The confirmation cohort has 40 scenarios
and the development cohort 4, and forward simulation with an
explicit pending table, an inverse trace over effective operations, and an
independent parser of the rendered text verify every label, and the
linearized twin is verified to reproduce the same final state. Seventeen
confirmation scenarios, fixed in index order before any answer, were
evaluated within the inference budget.

\subsection{Provenance and reproducibility}
PrefEval supplies preference content for the original core. The extensions
use newly constructed state ledgers, catalogues, and operational histories. Generation models supply dialogue wording, while executable rules determine
answers. We retain the construction seeds, scenario and condition identities,
programmatic labels, prompts, validation records, and raw evaluation responses.
API records include requested and returned model identifiers, provider,
collection date, and token usage. Repeated conditions remain linked to their
underlying scenario for paired analysis, and development and confirmation cohorts
are reported separately. The experiment code and CICM dataset are publicly available at
\url{https://github.com/JamesJunyuGuo/CICM}.

\FloatBarrier
\section{Test-time attention routing}
\label{app:test-time-repair}

We give the routing method, calibration, and controls behind
Table~\ref{tab:test-time-repair}.

\subsection{Identifying and attending to the current assignment}

The parser uses the dialogue, predefined variables and values, and tokenizer
offsets. It identifies the queried variable from the final question and scans
preceding user messages for updates. The latest valid update supplies the current value, and mentions of different values supply old-value positions.
Reminders containing both ``still thinking about'' and ``compare options''
are not updates, and assistant messages do not set the current value.
Saved answer labels and spans are used only to check the parser, not as its
inputs. This grammar-specific parser could itself answer the direct question.
The reminder baseline instead inserts its parsed value in an assistant message
before the query and lets the model answer.

Equation~\ref{eq:attention-routing} changes selected query heads only at the
final prompt position, after native scaling, masking, and any softcapping,
before softmax. Shared keys are first expanded to query heads in grouped-query
attention. Later cached decoding steps are not directly modified, although
effects can propagate. Weights, value vectors, and output logits are not
directly overwritten.

For adaptive routing, let $m$ be the calibrated target log-ratio of attention
on current versus old positions. Using the sets $C,O$ and scores $s_j$
defined in Section~\ref{sec:test-time-repair}, the existing log-ratio $\delta$
and the applied bias are
\begin{equation}
\delta=\log\sum_{j\in C}e^{s_j}-\log\sum_{j\in O}e^{s_j},
\qquad
\beta=\min\!\left(8,\max\!\left(0,\frac{m-\delta}{2}\right)\right).
\label{eq:adaptive-routing}
\end{equation}
The softmax denominator cancels: the bias increases the attention log-ratio
by $2\beta$, reaching $m$ unless capped. A sufficient ratio is left unchanged.
This does not fix answer probabilities or the total attention on these
positions relative to other context.
\subsection{Calibration and held-out evaluation}

The 1,200 CICM dialogues (Appendix~\ref{app:cicm-datasheet}) cover three
variables, seven values each, and six combinations of competitor distances.
A deterministic identifier hash assigns 40 of each cell's 200 dialogues to
calibration and 160 to testing. Adaptive calibration splits into 120 discovery
and 120 validation dialogues. Variables, vocabulary, and templates are not
held out.

With weights frozen, we differentiate the current-minus-strongest-old
first-token answer score with respect to a gate on each head's intervention,
at zero intervention ($m=1$, bias cap 8). Scores combine tokenization variants
by log-sum-exp. Heads are ranked by average gradient on discovery old-value
errors. First-token collisions are excluded from this objective, not from
full-response evaluation. The gradient measures local sensitivity, not each
head's independent correction effect. Validation tests the highest-ranked
positive-gradient heads in sets of size $1,2,4,8,16,32$ with
$m\in\{0.5,1,2,4\}$. We maximize positive accuracy gain while preserving
at least 95\% of correct answers, breaking ties toward fewer heads and smaller
$m$ (frozen settings: Table~\ref{tab:repair-estimates}).

Fixed-bias runs use eight heads and select $\beta\in\{0.5,1,2,4,8\}$ under
the same preservation constraint, without separating discovery and validation.
Qwen2.5-7B's eight heads are inherited from Stage B on the separate synthetic
assignment dataset, and the saved frozen head list
exactly matches that analysis, not the CICM attention comparison in
Section~\ref{sec:mechanism}. Qwen2.5-14B's
are ranked on calibration data by the failed-minus-correct difference in
log(old/current attention). Frozen configurations apply to every test dialogue regardless of baseline
correctness.

Checkpoints are instruction-tuned, and decoding is greedy with bfloat16 weights
and an eight-token limit. Runs use H100 GPUs  and batch size
12. Inference uses one decode, no gradients, and eager attention. 

\subsection{Retrieval estimates and controls}

Adaptive gains mainly correct old-value errors while preserving correct
answers (Table~\ref{tab:repair-estimates}). Value matching prioritizes current,
old, unused same-variable, other-variable, then other text. This is not
whole-string matching: answers containing both current and old values can
count as correct. Counting all mixed-value answers as incorrect leaves Qwen,
Llama-3.2-3B, and Gemma gains unchanged, while Llama-3.1-8B gains 46.82 points and
Mistral 76.88 points under this stricter rule.

Main-text accuracies weight dialogues equally. For uncertainty, dialogues
form 210 groups defined by variable, current value, overwrite count, and
competitor distances. We average paired correctness differences within
groups, then resample groups 2,000 times and average their means. These
equally weighted group estimates can differ from dialogue-weighted gains, and intervals condition on the frozen configuration, excluding head-selection
uncertainty.

\begin{table}[htbp]
\centering
\footnotesize
\setlength{\tabcolsep}{4pt}
\caption{Frozen routing settings and held-out retrieval outcomes. Counts show
old-value errors corrected and originally correct answers preserved.
Gain is the equally weighted group estimate (percentage points), with its
95\% bootstrap interval.}
\label{tab:repair-estimates}
\begin{tabular}{@{}lrrccr@{}}
\toprule
Model & Heads & Setting & Corrected & Preserved & Group gain [95\% interval] \\
\midrule
Qwen2.5-3B & 32 & $m=4$ & 678/711 & 156/158 & 78.37 [74.51, 82.21] \\
Llama-3.2-3B & 32 & $m=4$ & 514/541 & 266/266 & 64.60 [59.40, 69.63] \\
Llama-3.1-8B & 16 & $m=4$ & 395/467 & 400/400 & 46.90 [41.99, 51.93] \\
Mistral-7B-v0.3 & 32 & $m=2$ & 564/674 & 51/52 & 75.01 [70.33, 79.12] \\
Gemma-2-9B & 32 & $m=4$ & 407/455 & 452/452 & 45.06 [39.63, 50.99] \\
\addlinespace[2pt]
Qwen2.5-7B & 8 & $\beta=8$ & 208/487 & 361/361 & 31.80 [27.88, 35.76] \\
Qwen2.5-14B & 8 & $\beta=2$ & 13/491 & 383/389 & 1.17 [0.14, 2.24] \\
\bottomrule
\end{tabular}
\end{table}

Targeted routing outperforms random controls, and reversing it harms accuracy
(Table~\ref{tab:repair-controls}). We test 16 random-position sets matching
positive/negative token counts and excluding controlled values. Adaptive runs
also test 64 random-head sets matching counts per layer and excluding selected
heads. They use the same rule, but adaptive biases can differ: these are not
matched-norm controls. Reverse routing recomputes its bias in the opposite
direction. The reminder comparison is descriptive, without a separate paired
significance test against routing.

\begin{table}[htbp]
\centering
\footnotesize
\setlength{\tabcolsep}{6pt}
\caption{Control gains in retrieval accuracy (percentage points). Random
columns report means and empirical 95\% ranges across sampled sets, not
confidence intervals for paired effects. A dash means the control was not run.}
\label{tab:repair-controls}
\begin{tabular}{@{}lrrr@{}}
\toprule
Model & Random positions & Random heads & Reverse routing \\
\midrule
Qwen2.5-3B & $0.08\ [-0.27, 0.38]$ & $-9.98\ [-14.42, -1.28]$ & $-16.35$ \\
Llama-3.2-3B & $-0.17\ [-0.48, 0.07]$ & $-5.05\ [-12.19, 0.88]$ & $-27.08$ \\
Llama-3.1-8B & $-0.36\ [-0.69, 0.00]$ & $-2.52\ [-10.90, 3.05]$ & $-28.15$ \\
Mistral-7B-v0.3 & $0.03\ [-0.10, 0.21]$ & $-0.49\ [-1.50, 0.45]$ & $-3.65$ \\
Gemma-2-9B & $0.19\ [-0.07, 0.38]$ & $-3.04\ [-7.07, 1.32]$ & $-27.71$ \\
\addlinespace[2pt]
Qwen2.5-7B & $0.15\ [-0.42, 0.59]$ & --- & $-16.04$ \\
Qwen2.5-14B & $0.18\ [0.00, 0.38]$ & --- & $-1.77$ \\
\bottomrule
\end{tabular}
\end{table}

\subsection{Diagnostic checks and interpretation}
\label{app:repair-checks}

Gemma-2-9B increases accuracy from 47.08\% to 91.88\%, compared
with 82.81\% for the reminder, and preserves all 452 initially correct
answers (Table~\ref{tab:test-time-repair}).

 Among  960 examples, routing led to   a 44.79-point
gain in accuracy, 407/455 old-value corrections, and 452/452 correct answers preserved.
Thus the large Gemma gain persists in this sensitivity analysis.

These results support controllable selection, not unique heads or a shared
native cause of errors. Calibration uses labels and gradients, while parsing uses structured update knowledge. Differing model versions and protocols preclude
a scaling claim. Earlier alternatives were less successful: global per-head
steering did not beat random directions, fixed attenuation selected no
eligible nonzero change, and Qwen2.5-1.5B routing did not exceed the random-head
range. They motivated the adaptive method but are not further replications.

\FloatBarrier
\section{Behavioral robustness and task diversity}
\label{app:api-robustness}

These behavioral checks examine the interpretation of later mentions,
answer-generation settings, and broader update tasks. They use new API responses and
frozen paired inputs. The original mechanism and routing estimates retain
their original dialogues, and the checks below do not re-estimate those internal
measurements on the modified prompts.

\subsection{Paired CICM wording and placement tests}
\label{app:cicm-validity-results}

\paragraph{Design.}
For each of the original 1,200 dialogues, we compare five conditions.
The first keeps the original wording. The second adds an explicitly
historical reminder, ``For the record, [value] was an earlier choice for my
[slot], and is no longer current''. The third replaces the old value with
neutral text that retains its message position and whitespace word count.
The fourth asks for the most recent explicit update and specifies that later
mentions of earlier choices are not updates. The fifth combines the
historical wording with the explicit question. The current assignment,
other-variable updates, acknowledgments, and message count remain fixed.
This yields 6,000 condition rows per model. In original far-old dialogues,
the historical reminder and the neutral replacement are no-ops, and the
combined condition is identical to the explicit question, so exact duplicate
requests reuse the same response and are not independent trials.

The separate placement test uses the 600 original far-old dialogues. Both
conditions contain the same explicitly historical old-value mention and
explicit-update question. We swap that mention and its acknowledgment with
a value-free filler pair, preserving the complete multiset of message
contents and the exact positions of current and other-variable updates.
The historical mention is 2--8 messages from the question in the near
condition and 20--32 in the far condition. This produces 600 paired
comparisons per model with equal mention counts, and unlike the original
factorial, this contrast isolates placement from adding a reminder.

\paragraph{Execution and scoring.}
Qwen2.5-7B-Instruct and GPT-4o receive the same frozen conditions through
OpenRouter with temperature zero and a 16-token output cap. A balanced
screen covers 180 original dialogues and 90 placement pairs, and the full
design, designated as primary before inspecting screen effect estimates,
contains those items and is not a held-out replication. Identical cached
responses are reused only after checking input and request hashes. We
preserve returned model/provider metadata, token usage, raw responses,
termination reasons, and retries.

The primary score matches the entire normalized final answer against the
fixed value vocabulary, allowing an optional ``Answer:'' prefix. It does
not count an explanation as correct merely because it mentions the gold
value. This is a new evaluation, not a re-scoring of the original
gold-anywhere results. Old-value rate denotes old responses divided by all valid responses, while old-value share of errors uses only incorrect responses
as its denominator. Paired differences use the same source dialogues, and uncertainty conditions on these generated items. The stronger-model
comparison measures task solvability under these labels. A separate blinded
GPT-6 Astra review checks 50 generated-language extension dialogues
(Appendix~\ref{app:cicm-diversity}), and this is model-based semantic validation.

\begin{table}[htbp]
\centering
\small
\setlength{\tabcolsep}{4pt}
\caption{CICM wording and query ablations on the full source set. Each entry is current-value accuracy / old-value answer rate (\%), with 600 dialogues per source condition.}
\label{tab:revision-cicm-wording}
\begin{tabularx}{\textwidth}{@{}Xrrrr@{}}
\toprule
& \multicolumn{2}{c}{Qwen2.5-7B} & \multicolumn{2}{c}{GPT-4o} \\
\cmidrule(lr){2-3}\cmidrule(lr){4-5}
Variant & Near source & Far source & Near source & Far source \\
\midrule
Original & 12.3 / 84.8 & 65.7 / 14.5 & 79.7 / 19.0 & 94.5 / 2.3 \\
Explicit historical wording & 62.5 / 25.5 & 65.7 / 14.5 & 98.5 / 0.0 & 94.5 / 2.3 \\
Neutral replacement & 62.3 / 17.0 & 65.7 / 14.5 & 95.0 / 2.8 & 94.5 / 2.3 \\
Explicit-update query & 7.3 / 87.3 & 61.2 / 9.8 & 88.7 / 8.8 & 94.8 / 1.7 \\
Combined & 39.2 / 41.8 & 61.2 / 9.8 & 99.3 / 0.0 & 94.8 / 1.7 \\
\bottomrule
\end{tabularx}
\vspace{3pt}\parbox{\linewidth}{\footnotesize Near and far refer to the original source conditions, pooled over update counts and other-variable distances. Old-value rates use all 600 dialogues as the denominator. Historical wording and neutral replacement leave far-source prompts unchanged; combined equals the explicit-query variant there. These are wording ablations, separate from the matched placement test below.}
\end{table}

\paragraph{Wording matters, but does not account for all errors.}
All 14,400 condition rows have valid, untruncated responses, representing
10,800 distinct API requests after exact-prompt reuse. In the 600 near-source
dialogues, historical wording raises Qwen accuracy from 74/600 to 375/600
(a paired gain of 50.2 points, 95\% interval $[46.0,54.3]$), and neutral replacement gives 374/600. The explicit-update question alone does not rescue Qwen:
accuracy is 44/600, versus 235/600 with both treatments. GPT-4o scores
478/600 on original near-source prompts, 591/600 with historical wording,
and 596/600 with both treatments (Table~\ref{tab:revision-cicm-wording}).
The question-only rewrite instead improves GPT-4o to 532/600. Its effect is
therefore model-dependent. One possible explanation for Qwen is that mentioning
``earlier choices'' in the clarification makes historical alternatives more
salient, consistent with selection competition. The rewrite also changes
update and slot instructions, so this lexical hypothesis is not isolated by
the present comparison.

The originally showcased near-old/far-other cell contains 200 dialogues.
On these same sources, Qwen returns old values in 190/200 original prompts,
59/200 with historical wording, and 45/200 with neutral replacement.
GPT-4o scores 191/200 originally and 194/200 with historical wording, and the latter has no old-value answers. At $k=1$, across the three near-source
other-variable distances, Qwen accuracy rises from 27/120 to 96/120 under
historical wording, and GPT-4o from 97/120 to 117/120. Thus the low original
near-source baseline is sensitive to wording, while explicitly obsolete
values can still be returned.

\begin{table}[htbp]
\centering
\small
\setlength{\tabcolsep}{5pt}
\caption{Matched placement of an explicitly historical value. Each model answers 600 paired dialogues; the current update, other-variable updates, message count, and message-content multiset are fixed within a pair.}
\label{tab:revision-cicm-placement}
\begin{tabularx}{\textwidth}{@{}lXrrr@{}}
\toprule
Model & Answer type & Far (\%) & Near (\%) & Difference\\
\midrule
Qwen2.5-7B & Current value & 58.8 & 46.7 & -12.2 \\
 & Old value & 14.8 & 36.7 & $+21.8$ \\
\addlinespace
GPT-4o & Current value & 98.3 & 96.5 & $-1.8$ \\
 & Old value & 0.0 & 0.0 & $0.0$ \\
\bottomrule
\end{tabularx}
\vspace{3pt}\parbox{\linewidth}{\footnotesize Differences are near minus far. Intervals use 2,000 paired bootstrap resamples of source dialogues (seed 20260920). Both arms explicitly identify the repeated value as obsolete and ask for the latest explicit update. GPT-4o produces no old-value answers in either arm, yielding a degenerate empirical bootstrap interval; the Wilson upper endpoint for each 0/600 rate is 0.64\%.}
\end{table}

\paragraph{Equal-mention placement.}
In the separate matched test, Qwen returns old values on 89/600 far and
220/600 near prompts, a paired increase of 21.8 points
(95\% interval $[18.0,25.8]$). Accuracy falls from 353/600 to 280/600:
66 pairs change from incorrect to correct and 139 from correct to incorrect.
GPT-4o scores 590/600 and 579/600, respectively, with no old-value answers, and its 31 errors select another variable's value. The old-value placement
effect therefore persists in Qwen under explicit update semantics, while
the stronger model largely solves this controlled task.

\paragraph{Uncertainty and API provenance.}
Paired intervals use 2,000 source-dialogue bootstrap resamples, and they do not treat reused variant rows as independent observations.
All Qwen requests return \texttt{qwen/qwen-2.5-7b-instruct} from Phala.
GPT-4o requests return \texttt{openai/gpt-4o} through OpenAI or Azure.
Restricting GPT near-source comparisons to pairs routed through the same
provider gives historical, neutral, query-only, and combined accuracy gains
of 17.8, 14.9, 7.8, and 15.3 points ($n=304,308,295,301$), versus
18.8, 15.3, 9.0, and 19.7 points on all 600 pairs. These are descriptive
provider-subset checks, not randomized provider controls. Full counts by
update load, distance cell, variant, and provider accompany the saved
response-level records.

\FloatBarrier

\subsection{Output budget and concise reasoning}
\label{app:reasoning-results}

\paragraph{Design.}
We select 50 existing 80-line overwrite prompts at each
$k\in\{2,4,8\}$, using a fixed content-hash ordering.
Each receives a paired no-overwrite control: only earlier target-variable
names change to fresh names of the same character length. Values, line
positions, the final target assignment, and the query stay fixed. The
confirmation set contains 150 overwrite/control pairs, disjoint from the
10 pairs used to check the pipeline.

Qwen2.5-7B-Instruct is evaluated with the original integer-only instruction
at 16- and 512-token caps, and with a concise-reasoning instruction at the
same 512-token cap: ``Briefly identify the queried variable's most recent
assignment. Use at most one sentence of reasoning, then end with exactly
one line in the form \texttt{FINAL: <integer>}.'' These arms use temperature zero.
Reasoning versus the long direct-answer arm
compares the instruction at an equal output cap, while long versus short direct
answers isolates the cap with the instruction unchanged.

\paragraph{Answer extraction.}
Direct answers must be a complete integer, allowing a terminal period.
Reasoning answers must contain exactly one terminal \texttt{FINAL:}
integer, and numbers elsewhere in the explanation are ignored. During the
disjoint smoke check, five of 20 Qwen reasoning responses put this field
on the same line as the explanation. Before any confirmation calls, we
froze a parser that accepts this whitespace variation while rejecting
multiple fields and trailing text. Original strict-line scores and raw
responses remain archived, and strict-format sensitivity is reported
separately. API failures and truncations are distinguished from incorrect
completed answers. When a wrong value belongs both to an old target
assignment and another variable, we retain an ambiguous error category
rather than assigning it uniquely to old-value selection.

% Increasing Qwen's direct-answer budget leaves overwrite accuracy almost
% unchanged (87/150 correct at 16 tokens; 86/150 at 512 tokens).
% The concise-reasoning instruction raises it to 111/150, a paired gain over
% the budget-matched direct arm of 16.7 percentage points
% (95\% paired-bootstrap interval $[9.3,24.0]$).
% The gains at $k=2,4,8$ are respectively 16, 20, and 14 points, with intervals
% $[6,28]$, $[8,32]$, and $[0,28]$.
% All 39 remaining Qwen reasoning errors on overwrite prompts return an old
% value of the queried variable; at $k=8$, accuracy is 52\% versus 98\% on
% paired no-overwrite controls.

% Requiring the final field to occupy its own line would reduce Qwen reasoning
% scores to 63/150 on overwrite prompts and 123/150 on controls, compared with
% 111/150 and 147/150 under the delimiter-based scorer frozen after smoke.
% This sensitivity concerns line placement: the final integer is unchanged.

\paragraph{High-load operational-identifier updates.}
\label{app:cicm-highload}
A separate frozen confirmation set contains 36 fresh scenarios across six
domains (meeting rooms, delivery routes, storage bins, service desks, pickup
lockers, and document trays) and scalar or partial-record updates. Fixed
templates render 384 user--assistant pairs. The target receives $k$ distinct
old identifiers before its current value, and a later mention explicitly marks
an old identifier as obsolete. No-overwrite controls replace earlier target
writes with neutral text, while cross-record controls redirect them to another
record and field. Corresponding messages match whitespace-word counts,
although token counts can differ. The event ledger supplies current and old
answer labels. Before confirmation we froze 36 bases, five
loads $k\in\{4,16,64,128,256\}$, and scoring rules. Each setting comprises
540 condition rows and 396 unique prompts because identical controls share
responses. Table~\ref{tab:highload-compact} presents the highest-load endpoint, and the full grid and paired analyses remain in the experiment archive.

Qwen and GPT-4o use temperature zero with direct-32, direct-512, or concise
reasoning-512 outputs. Concise reasoning requests at most one explanatory
sentence and a terminal \texttt{FINAL:} identifier. GPT-5 and Gemini 3.1 Pro
Preview use medium reasoning and 8,192-token caps, while DeepSeek V4 Pro 0813 uses
high reasoning and 32,768 tokens. These API settings do not equate reasoning
computation. Direct answers score the normalized full reply, while reasoning arms score the terminal identifier. Protocol failures and truncations are recorded
separately. Fourteen Qwen rate-limited requests received one additional attempt
under a frozen recovery protocol. All completed, and no task answer was retried.
Exact request/model/provider metadata and all attempts are archived.

\begin{table}[htbp]
\centering\small
\setlength{\tabcolsep}{3pt}
\caption{High-load updates ($k=256$): 36 frozen scenarios per setting.
Controls report correct answers; overwrite columns distinguish current and old values.}
\label{tab:highload-compact}
\begin{tabular}{@{}llrrrr@{}}
\toprule
Model & Output, token cap & Current & Old & No overwrite & Cross-record \\
\midrule
Qwen2.5-7B & Direct, 32 & 0/36 & 36/36 & 34/36 & 36/36 \\
Qwen2.5-7B & Direct, 512 & 0/36 & 36/36 & 34/36 & 36/36 \\
Qwen2.5-7B & Concise, 512 & 0/36 & 32/36 & 36/36 & 33/36 \\
GPT-4o & Direct, 32 & 3/36 & 31/36 & 36/36 & 36/36 \\
GPT-4o & Direct, 512 & 3/36 & 32/36 & 36/36 & 36/36 \\
GPT-4o & Concise, 512 & 4/36 & 32/36 & 36/36 & 36/36 \\
GPT-5 & Reasoning, 8,192 & 36/36 & 0/36 & 36/36 & 36/36 \\
Gemini 3.1 Pro & Reasoning, 8,192 & 36/36 & 0/36 & 36/36 & 36/36 \\
DeepSeek V4 Pro & Reasoning, 32,768 & 34/36 & 2/36 & 35/36 & 35/36 \\
\bottomrule
\end{tabular}
\par\smallskip
\parbox{\linewidth}{\footnotesize All requested prompts remain in the denominator.
At this load, every overwrite response completes; remaining overwrite outcomes
are other answers (Qwen concise: 4; GPT-4o direct-32: 2; direct-512: 1).
DeepSeek's no-overwrite control has one unsuccessful provider finish;
all other control errors are other answers.}
\end{table}

At $k=256$, Qwen and GPT-4o each return 32/36 old values under concise reasoning.
GPT-4o solves every paired control in all three output settings. GPT-5 and
Gemini solve all tested loads through $k=256$, while DeepSeek returns two old values
at this endpoint. These results distinguish resistance to repeated explicit
updates from performance on the dependent histories in
Appendix~\ref{app:frontier-state-logs}.

\FloatBarrier
\subsection{Task and language diversity}
\label{app:cicm-diversity}

This experiment tests whether old-value errors extend beyond the original
preference templates. We evaluate 180 state ledgers spanning six domains and
three update operations: single-value replacement, partial-record updates,
and full-record replacement. Each ledger has four paired versions crossing
template or model-generated language with near or far placement of the same
explicitly obsolete mention. Within each placement pair, value-mention counts
and the current assignment's position are fixed. The query explicitly asks
for the state after all updates. Appendix~\ref{app:cicm-dialogue-data} gives
the construction, generation models, and validation procedure.

\paragraph{Evaluation.}
Qwen2.5-7B-Instruct, Llama-3.1-8B-Instruct, and GPT-4o receive all 720 prompts
through OpenRouter, with temperature zero and a 16-token output cap. All
2,160 responses complete without API errors or truncation. A frozen matcher
compares the complete normalized answer with the program-derived state and
separately identifies old target values and values of other fields.
Intervals use 2,000 bootstrap resamples of whole ledgers within domain,
update-load, and task strata, preserving the four paired versions.

\begin{table}[htbp]
\centering
\small
\setlength{\tabcolsep}{5pt}
\caption{Old-value responses on generated CICM dialogues. Each placement
contains the same 180 ledgers. The paired difference is near minus far in
percentage points; brackets give 95\% bootstrap intervals.}
\label{tab:cicm-diversity}
\begin{tabular}{@{}lrrr@{}}
\toprule
Model & Far & Near & Difference \\
\midrule
Qwen2.5-7B & 39/180 (21.7\%) & 50/180 (27.8\%) & $+6.1$ \\
Llama-3.1-8B & 7/180 (3.9\%) & 5/180 (2.8\%) & $-1.1$ \\
GPT-4o & 0/180 (0.0\%) & 0/180 (0.0\%) & $0.0$ \\
\bottomrule
\end{tabular}
\end{table}

\paragraph{Core finding.}
Qwen returns old values across all six domains and all three operations.
On generated language, moving the obsolete mention nearer the question
increases its old-value rate by 6.1 percentage points
(Table~\ref{tab:cicm-diversity}). The increase comes from record updates:
old-value counts rise from 14 to 20 for partial records and from 11 to 19
for full replacement, while single-value counts fall from 14 to 11. Template language gives a similar pooled increase
of 7.2\%. Llama makes fewer errors and shows no pooled increase, and GPT-4o answers all 720 prompts correctly. Thus old-value selection
extends beyond the original preference templates, while sensitivity to
mention placement varies across models and update operations.

\paragraph{Scoring and provider checks.}
Accepting one Qwen field-prefixed answer changes its generated-language
difference from 6.1 to 6.7 points. Accepting two Llama spelling variants
raises its four-condition accuracy from 697/720 to 699/720 without changing
old-value counts. Primary scores retain the frozen matcher. Qwen and GPT-4o
have equal reported input-token counts within every placement pair. Llama
uses several routed providers, and restricting generated-language pairs to the
76 sharing a provider still yields no positive placement effect
($-3.9$\%). These are descriptive sensitivity checks. Raw outputs,
provider and token metadata, and task/domain breakdowns are retained in the
reproducibility artifacts.

\FloatBarrier
\subsection{Updating constraints before making a decision}
\label{app:decision-transfer}

\paragraph{Task structure.}
We extend evaluation from current-field answers to decisions that depend on
several updated constraints. Two constructed task families use executable
ledgers and fixed language templates. In procurement dialogues, the assistant
maintains requirements for six projects and books a venue, caterer, and shuttle.
The components must agree on site and satisfy current capacity, accessibility,
dietary, walking-distance, rating, and budget constraints. The objective is
minimum total spend, with an explicit tie rule. A simpler query books only a
venue from the same state and catalogue.

In scheduling logs, an agent arranges two sessions under current durations,
attendance, equipment, room availability, shared-room, and break requirements.
It minimizes the second session's finish time, then room cost and explicit
secondary criteria. The simpler query schedules only the first session.
The log includes successful cached planning proposals tied to the directives
that produced them, and these are neither executed bookings nor new requirements.
A changed constraint can invalidate a proposal without replacing its text.
These are constructed dialogues and agent logs, not sampled deployment traces. Family and presentation covary, so their difference is not a surface-only test.

\paragraph{Scope and withdrawals.}
Each project inherits shared defaults until it has an active local override.
Local overrides take precedence over later shared defaults. Within a scope,
the latest active assignment to a field wins, and unmentioned fields persist.
A withdrawal deactivates only its named earlier assignment, including every
field that assignment changed. Remaining active directives determine the state, and withdrawals are never themselves withdrawn. These rules are stated explicitly
in each task. A ledger interpreter supplies the resolved state, and independent reverse-lookup replay and enumeration check the optimal actions.

\paragraph{Context and task complexity.}
Each scenario has a 24-directive history and a paired 96-directive history that
inserts 72 updates to other projects. The target's relevant directives, opaque
references, current state, and catalogue remain fixed. Each history supports
both a single-action and a composed-action query. Each query has three input
conditions: full history, an explicit final target-state snapshot, and that
snapshot padded to the history's whitespace-word count. The snapshots still
require solving the decision problem and never supply its answer. Padding consists of repeated instances of the word ``note'', not natural
distractor dialogue. It is therefore an artificial word-count control whose
response pathologies must be distinguished from historical interference.
Actual token counts are recorded rather than assumed equal. The longer-history contrast
therefore changes contextual interference, not the number of target updates.

The development cohort contains six scenarios. A separate fixed seed generates
24 evaluation scenarios, 12 per family, with no development prompt overlap.
There are 12 condition rows per scenario and 288 per model. Scenarios are
constructed to have a feasible unique optimum under the stated tie rule and
at least two distinct superseded optima during the relevant update history.
This programmatic enrichment precedes model answers. All scenarios are retained,
including models' successful cases. Identical snapshot requests may share a
cached response when their scoring contracts also agree, and request identities track reuse and separate repeated calls. They are not additional independent
scenarios.

\paragraph{Generation and scoring.}
GPT-5, Gemini 3.1 Pro Preview, DeepSeek V4 Pro 0813, and Claude Sonnet 5 use
high reasoning effort. DeepSeek has a 32,768-token completion cap and the other
models have 16,384, and these settings do not equate computation across families.
The final answer specifies component identifiers or room/time pairs, composing
an action rather than selecting a prewritten complete answer. Transport failures
receive at most two retries, and answer outcomes are never selectively retried.

A development-only formatting check shows that some models surround a correct
identifier with a short booking phrase. Before evaluation, we freeze a
conservative terminal-action parser that accepts identifiers, role labels,
punctuation, and a booking phrase naming the correct project. It rejects
negation, alternatives, conflicting projects, excess identifiers or times,
and missing terminal answers. All original strict-format scores are retained.
Refused and truncated responses remain separate from semantic errors.

Primary accuracy is the probability of returning the current optimal action
per requested prompt. We also report satisfaction of all current hard
constraints, feasible but nonoptimal decisions, and matches to a superseded
or another project's optimum. A historical-optimum match is consistent with obsolete constraints, but it does not establish which information caused the
answer or transfer the open-model circuit diagnosis. The primary paired
comparison is composed-task accuracy under long history minus its padded
snapshot control. Supporting comparisons vary task complexity and contextual
load. Intervals use 2,000 bootstrap resamples of whole scenarios within family,
retaining all 12 conditions together.

\paragraph{Decisions remain largely correct under full history.}
Tables~\ref{tab:decision-procurement}--\ref{tab:decision-paired} give all task,
context, and control cells. Under full history, optimal answers per requested
prompt are GPT-5 92/96, Gemini 95/96, DeepSeek 92/96, Sonnet 88/96. Every action accepted by the frozen
parser in these history conditions is optimal, and the remaining outcomes are
invalid terminal outputs, provider failures, or truncations, rather than
interpretable nonoptimal actions. Four GPT-5 history answers identify the
correct venue but add an unrequested date or wording outside the frozen
grammar. Their primary scores remain invalid-output outcomes, and they are
not evidence of obsolete-state use.

Gemini's three interpretable decision errors all occur in the long, composed
procurement task with the padded current-state control: two actions violate
current constraints and one is feasible but nonoptimal. One infeasible action
matches a historical optimum even though that input contains no history.
Such a match alone therefore cannot diagnose use of an obsolete state
in a composed-action task. The repeated-word control also produces invalid
outputs and truncations. Positive history-minus-padding differences cannot
be interpreted as a benefit of historical interference. These tasks extend
the tested action structure, but the observed history conditions do not show
a semantic decision-failure effect in the reasoning models. The compact and
padded controls, complete failure classes, and small number of independent
scenarios delimit this result.

Requested model
identifiers are \texttt{openai/gpt-5}, \texttt{google/gemini-3.1-pro-preview},
\texttt{deepseek/deepseek-v4-pro-0813}, and \texttt{anthropic/claude-sonnet-5}.
Returned identifiers, providers, token usage and every attempt are retained.

\begin{table}[htbp]
\centering
\small
\caption{Procurement dialogues: all 12 scenarios per cell. H denotes full history, S an explicit final-state snapshot, and P its word-count-matched padded version. Optimality counts successes per requested prompt. Unavailable includes invalid terminal output, provider failure/refusal, and truncation; it is separate from an infeasible or nonoptimal action. All interpretable history actions are optimal.}
\label{tab:decision-procurement}
\begin{tabular}{@{}llcc@{}}
\toprule
Context & Task & Optimal H / S / P & Unavailable H / S / P \\
\midrule
\multicolumn{4}{@{}l}{\textit{GPT-5}} \\
24 directives & Single & 11 / 12 / 12 & 1 / 0 / 0 \\
24 directives & Composed & 12 / 12 / 12 & 0 / 0 / 0 \\
96 directives & Single & 9 / 12 / 12 & 3 / 0 / 0 \\
96 directives & Composed & 12 / 12 / 12 & 0 / 0 / 0 \\
\addlinespace
\multicolumn{4}{@{}l}{\textit{Gemini 3.1 Pro Preview}} \\
24 directives & Single & 12 / 12 / 12 & 0 / 0 / 0 \\
24 directives & Composed & 11 / 12 / 12 & 1 / 0 / 0 \\
96 directives & Single & 12 / 12 / 1 & 0 / 0 / 11 \\
96 directives & Composed & 12 / 12 / 1 & 0 / 0 / 8 \\
\addlinespace
\multicolumn{4}{@{}l}{\textit{DeepSeek V4 Pro 0813}} \\
24 directives & Single & 12 / 12 / 12 & 0 / 0 / 0 \\
24 directives & Composed & 12 / 12 / 12 & 0 / 0 / 0 \\
96 directives & Single & 11 / 12 / 12 & 1 / 0 / 0 \\
96 directives & Composed & 12 / 12 / 12 & 0 / 0 / 0 \\
\addlinespace
\multicolumn{4}{@{}l}{\textit{Claude Sonnet 5}} \\
24 directives & Single & 10 / 12 / 12 & 2 / 0 / 0 \\
24 directives & Composed & 10 / 12 / 12 & 2 / 0 / 0 \\
96 directives & Single & 10 / 12 / 11 & 2 / 0 / 1 \\
96 directives & Composed & 12 / 12 / 7 & 0 / 0 / 5 \\
\addlinespace
\bottomrule
\end{tabular}
\end{table}

\begin{table}[htbp]
\centering
\small
\caption{Scheduling agent logs: all 12 scenarios per cell. H denotes full history, S an explicit final-state snapshot, and P its word-count-matched padded version. Optimality counts successes per requested prompt. Unavailable includes invalid terminal output, provider failure/refusal, and truncation; it is separate from an infeasible or nonoptimal action. All interpretable history actions are optimal.}
\label{tab:decision-scheduling}
\begin{tabular}{@{}llcc@{}}
\toprule
Context & Task & Optimal H / S / P & Unavailable H / S / P \\
\midrule
\multicolumn{4}{@{}l}{\textit{GPT-5}} \\
24 directives & Single & 12 / 12 / 12 & 0 / 0 / 0 \\
24 directives & Composed & 12 / 12 / 12 & 0 / 0 / 0 \\
96 directives & Single & 12 / 12 / 12 & 0 / 0 / 0 \\
96 directives & Composed & 12 / 12 / 12 & 0 / 0 / 0 \\
\addlinespace
\multicolumn{4}{@{}l}{\textit{Gemini 3.1 Pro Preview}} \\
24 directives & Single & 12 / 12 / 12 & 0 / 0 / 0 \\
24 directives & Composed & 12 / 12 / 12 & 0 / 0 / 0 \\
96 directives & Single & 12 / 12 / 11 & 0 / 0 / 1 \\
96 directives & Composed & 12 / 12 / 11 & 0 / 0 / 1 \\
\addlinespace
\multicolumn{4}{@{}l}{\textit{DeepSeek V4 Pro 0813}} \\
24 directives & Single & 11 / 12 / 12 & 1 / 0 / 0 \\
24 directives & Composed & 10 / 11 / 12 & 2 / 1 / 0 \\
96 directives & Single & 12 / 12 / 12 & 0 / 0 / 0 \\
96 directives & Composed & 12 / 12 / 12 & 0 / 0 / 0 \\
\addlinespace
\multicolumn{4}{@{}l}{\textit{Claude Sonnet 5}} \\
24 directives & Single & 11 / 12 / 12 & 1 / 0 / 0 \\
24 directives & Composed & 12 / 12 / 12 & 0 / 0 / 0 \\
96 directives & Single & 11 / 12 / 12 & 1 / 0 / 0 \\
96 directives & Composed & 12 / 12 / 12 & 0 / 0 / 0 \\
\addlinespace
\bottomrule
\end{tabular}
\end{table}

\begin{table}[htbp]
\centering
\small
\caption{Paired optimal-action differences on 24 scenarios (percentage points, 95\% whole-scenario bootstrap intervals). H, S, and P are defined above. The first two columns use the 96-directive composed task; the final column subtracts the single-task H$-$P difference from the composed-task H$-$P difference at that length. All-requested denominators retain no-answer outcomes.}
\label{tab:decision-paired}
\begin{tabular}{@{}lccc@{}}
\toprule
Model & H $-$ P & H $-$ S & Change with composition \\
\midrule
GPT-5 & 0.0 [0.0, 0.0] & 0.0 [0.0, 0.0] & 12.5 [0.0, 25.0] \\
Gemini 3.1 Pro Preview & 50.0 [37.5, 62.5] & 0.0 [0.0, 0.0] & 0.0 [-16.7, 16.7] \\
DeepSeek V4 Pro 0813 & 0.0 [0.0, 0.0] & 0.0 [0.0, 0.0] & 4.2 [0.0, 12.5] \\
Claude Sonnet 5 & 20.8 [8.3, 33.3] & 0.0 [0.0, 0.0] & 29.2 [12.5, 50.0] \\
\bottomrule
\end{tabular}
\end{table}

\paragraph{Exploratory extensions with explicit dependencies.}
Two additional task-development pilots test structures absent from the initial
families. A coding-agent pilot has 16 dependent tools, 18 versioned files, and
two branches. A successful cached result is reusable only if its captured file
versions and every recursively captured input artifact remain valid for the
target branch. Late completion does not refresh captured inputs, and a newer
stale result does not erase an older valid cache. The task requests the minimum
set of tools to rerun. Six constructed scenarios have both full-history and
current-files-plus-cache-input conditions, and the latter supplies neither freshness
labels nor the execution plan. With high reasoning and a 16,384-token cap,
GPT-5 and Gemini each recover all six minimum plans in both conditions. GPT-5
sometimes omits the ``T'' prefix from tool identifiers. An explicitly post-hoc
equivalence check maps unique bare integers to the fixed tool list, and the original strict scores are retained, and these formatting differences are not missing
steps.

A subsequent procurement pilot makes decisions jointly dependent through five
supplier roles, ten incompatible pairs, and 30 stackable conditional credits.
The objective is maximum total quality under the updated global budget, then
minimum net cost and a stated identifier tie rule. Three scenarios are selected
programmatically before answers: each has at least 15 feasible plans, its
quality optimum differs from its cheapest plan, and credits enable strictly
higher quality than any no-credit plan. Each has single-venue and composed-plan
queries, with full history and explicit current-state controls. GPT-5 returns
11/12 optimal answers, and the remaining composed-history request reaches its
16,384-token cap, while its paired snapshot is correct. Gemini returns 12/12
optimal answers. Every completed plan is optimal. These adaptive exploratory
pilots identify task structures and resource demands, but they do not establish a
reliable frontier-model old-state failure. Their compact controls are not
length-matched, and the small samples do not estimate deployment prevalence.

\FloatBarrier
\subsection{Frontier state reconstruction on complex operational histories}
\label{app:frontier-state-logs}

\paragraph{From overwrites to dependent state updates.}
We test frontier reasoning on current-state reconstruction through long
sequences of dependent operations. Constructed warehouse-dispatch and build-release logs contain
16 slots holding immutable parcel or artifact identifiers. A handoff atomically
cycles the \emph{current} contents of two, three, or four slots, and the result of one operation becomes the input to later operations. Only committed production
operations affect the queried state, while sandbox and aborted operations do not.
The prompt supplies these rules, the initial state, all 2,048 operations in
execution order, and a query for the single identifier to dispatch from one
slot. There are no hidden updates. This uses permutation composition, an
established state-tracking model system \citep{li2025statetracking}, in two
operational settings sharing the same transition rules.

\paragraph{Protocol.}
Task development used four pilot scenarios. Before collecting confirmation
answers, we froze 40 fresh scenarios, each
paired with a programmatically resolved current-state snapshot. All scenarios
are included. Forward replay, inverse ancestry, and an independent parser of
the rendered log verify the answers. We evaluate the model directly through
the API, without executable tools, external memory, or a context-management
harness. The snapshot is a separate solvability control: it supplies the final
slot contents and removes the need to reconstruct them from history. It is
not matched for length or computation, and does not evaluate a learned or
deployed harness. Claude models were evaluated later on the identical frozen
prompts through the native Anthropic API (official SDK, adaptive thinking,
128,000-token output cap, no tools).

\paragraph{Frontier reasoning models fail on full histories.}
GPT-5.6 Sol with high reasoning answers 9/40 full-history queries correctly,
compared with 40/40 paired snapshots (Table~\ref{tab:frontier-state-logs}).
All 80 requests yield completed semantic answers, and none is truncated, refused,
invalid, or lost to a provider error. Claude Opus 4.8 
answers 18/40 correctly: 13 answers name an old occupant of the queried slot
and 9 responses stop at the 128,000-token output cap without an answer. Among completed answers, accuracy is 18/31 (58.1\%, Wilson [40.8, 73.6]\%). Its 40
paired snapshots are all correct. Claude Sonnet~5 stops at the output cap on
every pilot history  (0/4 completed, snapshots 4/4), and its summarized reasoning describes a forward simulation of all sixteen
slots that reaches only part of the log, so its outcome is an output-availability
result, not a semantic one, and no confirmation was run. These are failures to
identify the current operational state despite having the complete history
and explicit transition rules. The contrast shows that using a supplied current
state is substantially easier than reconstructing it through these updates.

\begin{table}[htbp]
\centering
\small
\caption{Current-state accuracy on complex operational histories. The primary
panels use the same 40 fresh paired scenarios, and brackets give Wilson 95\%
intervals over all requests. Truncation at the output cap counts as no answer.
The exploratory panel reuses four separate pilot scenarios and is reported as
counts. Sol and DeepSeek use high reasoning.}
\label{tab:frontier-state-logs}
\begin{tabularx}{\linewidth}{@{}Xrr@{}}
\toprule
Model and cohort & Full history & Current-state snapshot \\
\midrule
GPT-5.6 Sol, confirmation & 9/40 & 40/40 \\
Accuracy [95\% interval] & 22.5\% [12.3, 37.5] & 100\% [91.2, 100] \\
\addlinespace
Claude Opus 4.8, confirmation & 18/40 (9 truncated) & 40/40 \\
Accuracy [95\% interval] & 45.0\% [30.7, 60.2] & 100\% [91.2, 100] \\
Completed-answer accuracy & 18/31 (58.1\%) & 40/40 (100\%) \\
\addlinespace
GPT-5.6 Sol, exploratory & 0/4 & 4/4 \\
DeepSeek V4 Pro 0813, exploratory & 2/4 & 4/4 \\
Claude Opus 4.8, exploratory & 0/4 & 4/4 \\
Claude Sonnet 5, exploratory & 0/4 (4 truncated) & 4/4 \\
\bottomrule
\end{tabularx}
\end{table}

\paragraph{Dependency-aware logs.}
To separate cross-line dependency from raw length, a second tier keeps the
same state machine and surfaces but logs some handoffs as \texttt{pending}:
such an operation executes only when a later line
\texttt{resolve(seq=N); result=committed} refers to it, on the slot contents
current at that line, while \texttt{result=aborted} or no resolution cancels it.
Each 1,024-line scenario has three conditions: the deferred history, a
linearized control with the identical effective operation sequence at the
identical line numbers but every effect stated on its own line, and the
resolved snapshot (construction in Appendix~\ref{app:cicm-log-data}).
Opus 4.8  was evaluated on a pre-specified prefix of 17 of
40 frozen scenarios (Table~\ref{tab:deferred-logs}): 12/17 deferred histories
are correct against 16/17 linearized controls and 17/17 snapshots, with no truncation, and five scenarios fail only in the deferred form and one only in the
linearized form (exact McNemar $p=0.22$). Four separate development scenarios
give 2/4, 4/4, and 4/4. Every error names an old occupant of the queried slot.
The remaining 23 frozen scenarios were not requested, and this exploratory comparison reports the complete pre-specified prefix.

\begin{table}[htbp]
\centering
\small
\caption{Dependency-aware logs: Claude Opus 4.8 on 1,024-line
scenarios with pending operations resolved by later lines. The linearized
control has the same effective operations at the same line numbers without
cross-line dependency, and brackets give Wilson 95\% intervals.}
\label{tab:deferred-logs}
\begin{tabularx}{\linewidth}{@{}Xrrr@{}}
\toprule
Cohort & Deferred history & Linearized control & Snapshot \\
\midrule
Confirmation prefix (17 scenarios) & 12/17 & 16/17 & 17/17 \\
Accuracy [95\% interval] & 70.6\% [46.9, 86.7] & 94.1\% [73.0, 99.0] & 100\% [81.6, 100] \\
\addlinespace
Development (4 scenarios) & 2/4 & 4/4 & 4/4 \\
\bottomrule
\end{tabularx}
\end{table}

\paragraph{Settings and supporting checks.}
 Sol uses OpenRouter model
\texttt{openai/gpt-5.6-sol}, fixed provider OpenAI, high reasoning, and a
128,000-token output cap. Its longest confirmation completion is 34,918 tokens, and maximum input length is 51,836 tokens. Claude models use the native Anthropic
API with returned identifiers \texttt{claude-opus-4-8} and
\texttt{claude-sonnet-5}, adaptive thinking with summarized display, a 128,000-token output
cap, and no sampling parameters, and history prompts are 67,181--67,359 Anthropic
tokens. Completed Opus answers use 35,631--103,140 output tokens (median
43,778), and all 31 state a backward trace of the queried slot, while six of the nine truncated responses instead begin a full forward simulation of all sixteen
slots. Thinking summaries are stored but never scored. The exploratory DeepSeek panel uses
\texttt{deepseek/deepseek-v4-pro-0813}, fixed provider Wafer, high reasoning,
temperature zero, and a 262,144-token cap. Sol and DeepSeek pilot answers complete normally. On those four pilot
scenarios, Sol and DeepSeek each answer 4/4 correctly
when every commit is cancelled and 4/4 when an accurate checkpoint is inserted
before the final 512 operations. For Sol, cancelling commits preserves input
bytes and token counts, and the checkpoint retains the earlier log and adds 147
tokens. Sol also answers 4/4 when irrelevant operations are removed and 4/4
when they are cancelled while preserving the complete necessary update chain.
These are exploratory diagnostics on reused cases. A corresponding 40-scenario
relevance study is incomplete: 8/40 compact and 10/40 masked requests returned,
all correctly, before service-credit exhaustion. The other 62 outcomes are
unavailable and contribute no accuracy estimate, and no completed recovery result
is included here.

\paragraph{Auditable failure examples.}
Figures~\ref{fig:frontier-log-warehouse} and~\ref{fig:frontier-log-build}
show Sol's first failed confirmation scenario by index in each setting, and Figures~\ref{fig:claude-log-build-agent-0019}
and~\ref{fig:claude-log-warehouse-0002} show Opus 4.8's first failed
build-release scenario on the 2,048-operation tier and its first failed
warehouse scenario on the dependency-aware tier, selected by the same rule. Input
and answer excerpts are verbatim, apart from line wrapping and explicitly
marked omissions. The model received all 16 initial entries and the complete log: 2,048
operations in the primary tier or 1,024 lines in the deferred tier.
Green boxes contain post-hoc programmatic checks, which were not supplied in
the full-history condition. The violet boxes reproduce visible answers, not
hidden reasoning traces. Full prompts, responses, model/provider metadata,
and input hashes are retained with the reproducibility artifacts.

\begin{figure}[htbp]
\centering
\tcbset{breakable=false}
\begin{tcolorbox}[enhanced,breakable=false,colback=neutralfill,colframe=neutralink,boxrule=0.7pt,arc=1mm,fonttitle=\bfseries\small,title={Warehouse dispatch --- handoff\_0000}]
\tcbset{breakable=false}
\small
\textbf{Task rules (summary).} Move current slot contents simultaneously
around each listed cycle. Only committed production operations change state.
\medskip\par
\textbf{Initial production slots (excerpt).}\par
\texttt{S08: P2070}\par
\texttt{S11: P6443}\par
\emph{[Other initial entries and earlier log entries omitted.]}\par
\medskip
\textbf{Execution log (excerpt).}\par
\texttt{2037 handoff(cycle=[S01, S09, S08, S15]);}\par
\texttt{branch=sandbox; status=committed}\par
\emph{[Intervening entries omitted.]}\par
\texttt{2041 handoff(cycle=[S11, S01, S07, S12]);}\par
\texttt{branch=production; status=committed}\par
\texttt{2042 handoff(cycle=[S10, S09, S13, S01]);}\par
\texttt{branch=production; status=committed}\par
\emph{[Entries 2043--2048 omitted.]}\par
\medskip
\textbf{Query (verbatim).} The next dispatch must use the item currently in S01. Which P identifier should the agent dispatch? Reason as needed. End with exactly one line in the form FINAL: <one requested P identifier>.
\end{tcolorbox}
\begin{tcolorbox}[enhanced,breakable=false,colback=stalefill,colframe=staleink,boxrule=0.8pt,arc=1mm,fonttitle=\bfseries\small,title={Sol answer to the complete history (verbatim)}]
\tcbset{breakable=false}
\small
Tracing S01 backward through committed production cycles leads to initial slot S11, which held P6443.\par
FINAL: P6443
\end{tcolorbox}
\begin{tcolorbox}[enhanced,breakable=false,colback=currentfill,colframe=currentink,boxrule=0.8pt,arc=1mm,fonttitle=\bfseries\small,title={Verified current state (post-hoc annotation)}]
\tcbset{breakable=false}
\small
Immediately before operation 2042, \texttt{S13} contains \texttt{P2070}. That committed production handoff moves it to \texttt{S01}; no later committed production operation changes this slot. Full inverse replay traces \texttt{S01} to initial \texttt{S08}, not \texttt{S11}. The correct dispatch is \textbf{\texttt{P2070}}.
\end{tcolorbox}
\caption{Warehouse dispatch failure with high reasoning. The model receives the
complete history and returns an incorrect dispatch identifier. Selected log
entries and the verified final transition explain the task; omitted entries
remain part of the evaluated input. The paired snapshot answer is correct.}
\label{fig:frontier-log-warehouse}
\end{figure}

\begin{figure}[htbp]
\centering
\tcbset{breakable=false}
\begin{tcolorbox}[enhanced,breakable=false,colback=neutralfill,colframe=neutralink,boxrule=0.7pt,arc=1mm,fonttitle=\bfseries\small,title={Build-release agent --- handoff\_0001}]
\tcbset{breakable=false}
\small
\textbf{Task rules (summary).} Move current slot contents simultaneously
around each listed cycle. Only committed production operations change state.
\medskip\par
\textbf{Initial production slots (excerpt).}\par
\texttt{S01: P1740}\par
\texttt{S03: P6022}\par
\emph{[Other initial entries and earlier log entries omitted.]}\par
\medskip
\textbf{Execution log (excerpt).}\par
\texttt{2046 handoff(cycle=[S09, S11, S15]);}\par
\texttt{branch=production; status=committed}\par
\texttt{2047 handoff(cycle=[S07, S16, S06]);}\par
\texttt{branch=production; status=committed}\par
\texttt{2048 handoff(cycle=[S04, S11, S06]);}\par
\texttt{branch=sandbox; status=committed}\par
\medskip
\textbf{Query (verbatim).} The next dispatch must use the item currently in S11. Which P identifier should the agent dispatch? Reason as needed. End with exactly one line in the form FINAL: <one requested P identifier>.
\end{tcolorbox}
\begin{tcolorbox}[enhanced,breakable=false,colback=stalefill,colframe=staleink,boxrule=0.8pt,arc=1mm,fonttitle=\bfseries\small,title={Sol answer to the complete history (verbatim)}]
\tcbset{breakable=false}
\small
Reverse-tracing S11 through only committed production cycles leads to initial slot S03, containing P6022.\par
FINAL: P6022
\end{tcolorbox}
\begin{tcolorbox}[enhanced,breakable=false,colback=currentfill,colframe=currentink,boxrule=0.8pt,arc=1mm,fonttitle=\bfseries\small,title={Verified current state (post-hoc annotation)}]
\tcbset{breakable=false}
\small
Immediately before operation 2046, \texttt{S09} contains \texttt{P1740}. That committed production handoff moves it to \texttt{S11}; no later committed production operation changes this slot. Full inverse replay traces \texttt{S11} to initial \texttt{S01}, not \texttt{S03}. The correct dispatch is \textbf{\texttt{P1740}}. Operation 2048 is a sandbox operation and leaves production unchanged.
\end{tcolorbox}
\caption{Build-release agent failure with high reasoning. The model receives the
complete history and returns an incorrect dispatch identifier. Selected log
entries and the verified final transition explain the task; omitted entries
remain part of the evaluated input. The paired snapshot answer is correct.}
\label{fig:frontier-log-build}
\end{figure}

\begin{figure}[htbp]
\centering
\tcbset{breakable=false}
\begin{tcolorbox}[enhanced,breakable=false,colback=neutralfill,colframe=neutralink,boxrule=0.7pt,arc=1mm,fonttitle=\bfseries\small,title={Build-release agent --- handoff\_0019}]
\tcbset{breakable=false}
\small
\textbf{Task rules (summary).} Move current slot contents simultaneously
around each listed cycle. Only committed production operations change state.
\medskip\par
\textbf{Initial production slots (excerpt).}\par
\texttt{S11: P4052}\par
\texttt{S16: P7085}\par
\emph{[Other initial entries and earlier log entries omitted.]}\par
\medskip
\textbf{Execution log (excerpt).}\par
\texttt{0514 handoff(cycle=[S06, S11, S14, S08]);}\par
\texttt{branch=production; status=committed}\par
\emph{[Intervening entries omitted.]}\par
\texttt{2044 handoff(cycle=[S05, S01, S15, S12]);}\par
\texttt{branch=production; status=committed}\par
\texttt{2045 handoff(cycle=[S11, S14, S04, S15]);}\par
\texttt{branch=production; status=committed}\par
\texttt{2046 handoff(cycle=[S04, S06, S12, S10]);}\par
\texttt{branch=production; status=aborted}\par
\texttt{2047 handoff(cycle=[S01, S03, S06, S13]);}\par
\texttt{branch=production; status=committed}\par
\texttt{2048 handoff(cycle=[S14, S09]);}\par
\texttt{branch=production; status=committed}\par
\medskip
\textbf{Query (verbatim).} The next dispatch must use the item currently in S01. Which P identifier should the agent dispatch? Reason as needed. End with exactly one line in the form FINAL: <one requested P identifier>.
\end{tcolorbox}
\begin{tcolorbox}[enhanced,breakable=false,colback=stalefill,colframe=staleink,boxrule=0.8pt,arc=1mm,fonttitle=\bfseries\small,title={Opus 4.8 answer to the complete history (final lines, verbatim)}]
\tcbset{breakable=false}
\small
\emph{[3320 characters of the model's stated trace omitted; the full answer is in the raw record.]}\par
After op 0003, target = S16, and no earlier committed production op touches S16.\par
The initial content of S16 is P7085.\par
FINAL: P7085\par
\end{tcolorbox}
\begin{tcolorbox}[enhanced,breakable=false,colback=currentfill,colframe=currentink,boxrule=0.8pt,arc=1mm,fonttitle=\bfseries\small,title={Verified current state (post-hoc annotation)}]
\tcbset{breakable=false}
\small
Immediately before line 2047, \texttt{S13} contains \texttt{P4052}. That committed production handoff moves it to \texttt{S01}; no later committed production operation changes this slot. Full inverse replay traces \texttt{S01} to initial \texttt{S11}, not \texttt{S16}. The correct dispatch is \textbf{\texttt{P4052}}. The answered identifier \texttt{P7085} last left \texttt{S01} at line 2044.
 The model's stated backward trace matches the verified chain for its 210 most recent transitions; at operation 514 (\texttt{handoff(cycle=[S06, S11, S14, S08])}) the tracked slot is \texttt{S08}, whose predecessor in the cycle is \texttt{S14}, but the trace moves to \texttt{S08}.
\end{tcolorbox}
\caption{Build-release agent failure (Opus 4.8, high effort, 40-scenario confirmation). The model receives the complete history and returns an incorrect dispatch identifier. Selected log entries and the verified transitions explain the task; omitted entries remain part of the evaluated input. The paired snapshot answer is correct.}
\label{fig:claude-log-build-agent-0019}
\end{figure}

\begin{figure}[htbp]
\centering
\tcbset{breakable=false}
\begin{tcolorbox}[enhanced,breakable=false,colback=neutralfill,colframe=neutralink,boxrule=0.7pt,arc=1mm,fonttitle=\bfseries\small,title={Warehouse dispatch --- deferred\_0002}]
\tcbset{breakable=false}
\small
\textbf{Task rules (summary).} Move current slot contents simultaneously
around each listed cycle. Only committed production operations change state.
A \texttt{pending} operation executes only when a later \texttt{resolve(seq=N); result=committed} line refers to it, on the contents current at that line.
\medskip\par
\textbf{Initial production slots (excerpt).}\par
\texttt{S02: P7238}\par
\texttt{S07: P7888}\par
\emph{[Other initial entries and earlier log entries omitted.]}\par
\medskip
\textbf{Execution log (excerpt).}\par
\texttt{0876 handoff(cycle=[S08, S09, S14, S12]);}\par
\texttt{branch=production; status=pending}\par
\emph{[Intervening entries omitted.]}\par
\texttt{1004 handoff(cycle=[S07, S01, S15]);}\par
\texttt{branch=production; status=committed}\par
\texttt{1005 handoff(cycle=[S13, S14]);}\par
\texttt{branch=production; status=committed}\par
\texttt{1006 handoff(cycle=[S01, S07]);}\par
\texttt{branch=production; status=pending}\par
\texttt{1007 resolve(seq=0876); result=committed}\par
\texttt{1008 handoff(cycle=[S05, S08, S13]);}\par
\texttt{branch=production; status=pending}\par
\texttt{1009 handoff(cycle=[S16, S08]);}\par
\texttt{branch=production; status=pending}\par
\texttt{1010 handoff(cycle=[S02, S09, S07]);}\par
\texttt{branch=production; status=pending}\par
\emph{[Entries after the excerpt omitted (14).]}\par
\medskip
\textbf{Query (verbatim).} The next dispatch must use the item currently in S14. Which P identifier should the agent dispatch? Reason as needed. End with exactly one line in the form FINAL: <one requested P identifier>.
\end{tcolorbox}
\begin{tcolorbox}[enhanced,breakable=false,colback=stalefill,colframe=staleink,boxrule=0.8pt,arc=1mm,fonttitle=\bfseries\small,title={Opus 4.8 answer to the complete history (final lines, verbatim)}]
\tcbset{breakable=false}
\small
\emph{[477 characters of the model's stated trace omitted; the full answer is in the raw record.]}\par
Checking all subsequent executing operations (lines 0008–1024): none of the committed operations or committed resolutions after that point touch S14 (line 1021's `[S10,S14,S12,S02]` remained pending and was never resolved).\par
Therefore S14 currently holds P7888.\par
FINAL: P7888\par
\end{tcolorbox}
\begin{tcolorbox}[enhanced,breakable=false,colback=currentfill,colframe=currentink,boxrule=0.8pt,arc=1mm,fonttitle=\bfseries\small,title={Verified current state (post-hoc annotation)}]
\tcbset{breakable=false}
\small
Immediately before line 1007, \texttt{S09} contains \texttt{P7238}. Line 1007 resolves the pending production handoff 0876 (\texttt{handoff(cycle=[S08, S09, S14, S12])}) as committed, which executes at that line and moves it to \texttt{S14}; no later executed operation changes this slot. Full inverse replay traces \texttt{S14} to initial \texttt{S02}. The correct dispatch is \textbf{\texttt{P7238}}. The answered identifier \texttt{P7888} last left \texttt{S14} at line 848.
\end{tcolorbox}
\caption{Warehouse dispatch failure (Opus 4.8, high effort, deferred-log tier, 17-scenario prefix). The model receives the complete history and returns an incorrect dispatch identifier. Selected log entries and the verified transitions explain the task; omitted entries remain part of the evaluated input. The paired snapshot answer is correct.}
\label{fig:claude-log-warehouse-0002}
\end{figure}

\end{document}